\documentclass{article}
\usepackage[T1]{fontenc}
\usepackage[utf8]{inputenc}
\usepackage{color}
\usepackage{multirow}
\usepackage{amsmath}
\usepackage{amsthm}
\usepackage{graphicx}
\usepackage[numbers]{natbib}
\PassOptionsToPackage{normalem}{ulem}
\usepackage{ulem}
\usepackage{microtype}
\usepackage[unicode=true,
 bookmarks=false,
 breaklinks=false,pdfborder={0 0 1},backref=false,colorlinks=false]
 {hyperref}

\makeatletter

\providecommand{\tabularnewline}{\\}

\theoremstyle{plain}
\newtheorem{thm}{\protect\theoremname}
\theoremstyle{remark}
\newtheorem{rem}[thm]{\protect\remarkname}
\theoremstyle{plain}
\newtheorem{assumption}[thm]{\protect\assumptionname}
\theoremstyle{plain}
\newtheorem{lem}[thm]{\protect\lemmaname}

\usepackage{iclr2023_conference}
\iclrfinalcopy
\usepackage{times}

\usepackage{url}
\usepackage{amsfonts}
\usepackage{nicefrac}
\allowdisplaybreaks[4]
\title{ASGNN: Graph Neural Networks with \\ Adaptive Structure}

\author{%
  Zepeng Zhang\textsuperscript{1}, Songtao Lu\textsuperscript{2}, Zengfeng Huang\textsuperscript{3}, Ziping Zhao\textsuperscript{1}\\
  \textsuperscript{1}ShanghaiTech University, \textsuperscript{2}IBM Research, \textsuperscript{3}Fudan University\\
  \texttt{zhangzp1@shanghaitech.edu.cn,songtao@ibm.com,huangzf@fudan.edu.cn,}\\
\texttt{zhaoziping@shanghaitech.edu.cn}
}

\usepackage{graphics}

\providecommand{\remarkname}{Remark}
\providecommand{\theoremname}{Theorem}

\makeatother

\providecommand{\assumptionname}{Assumption}
\providecommand{\lemmaname}{Lemma}
\providecommand{\remarkname}{Remark}
\providecommand{\theoremname}{Theorem}

\begin{document}
\maketitle 
\begin{abstract}
The graph neural network (GNN) models have presented impressive achievements
in numerous machine learning tasks. However, many existing GNN models
are shown to be vulnerable to adversarial attacks, which creates a
stringent need to build robust GNN architectures. In this work, we
propose a novel interpretable message passing scheme with adaptive
structure (ASMP) to defend against adversarial attacks on graph structure.
Layers in ASMP are derived based on optimization steps that minimize
an objective function that learns the node feature and the graph structure
simultaneously. ASMP is adaptive in the sense that the message passing
process in different layers is able to be carried out over dynamically
adjusted graphs. Such property allows more fine-grained handling of
the noisy (or perturbed) graph structure and hence improves the robustness.
Convergence properties of the ASMP scheme are theoretically established.
Integrating ASMP with neural networks can lead to a new family of
GNN models with adaptive structure (ASGNN). Extensive experiments
on semi-supervised node classification tasks demonstrate that the
proposed ASGNN outperforms the state-of-the-art GNN architectures
in terms of classification performance under various adversarial attacks. 
\end{abstract}

\section{Introduction}

Graphs, or networks, are ubiquitous data structures in many fields
of science and engineering \citep{newman2018networks}, like molecular
biology, computer vision, social science, financial technology, etc.
In the past few years, due to its appealing capability of learning
representations through message passing over the graph structure,
graph neural network (GNN) models have become popular choices for
processing graph-structured data and have achieved astonishing success
in various applications \citep{kipf2017semi,bronstein2017geometric,wu2020comprehensive,zhou2020graph,GNNBook2022}.
However, existing GNN backbones such as the graph convolutional network
(GCN) \citep{kipf2017semi} and the graph attention network \citep{velivckovic2018graph}
are shown to be extremely vulnerable to carefully designed adversarial
attacks on the graph structure \citep{sun2018adversarial,jin2021adversarial,GNNBook-ch8-gunnemann}.
With unnoticeable malicious manipulations of the graph, the performance
of GNNs significantly drops and may even be worse than the performance
of a simple baseline that ignores all the relational information among
data feature \citep{dai2018adversarial,zugner2018adversarial,zugner2019adversarial,zhang2020gnnguard}.
With the increasing deployments of GNN models in various real-world
applications, it is of vital importance to ensure their reliability
and robustness, especially in scenarios, such as medical diagnosis
and credit scoring, where a deflected model can lead to dramatic consequences
\citep{GNNBook-ch8-gunnemann}.

To improve the robustness of GNNs with a potentially noisy graph structure
input, a natural idea is to ``purify'' the given graph structure.
Existing work in this line can be roughly classified into two categories.
The first category of robustifying GNNs can be viewed as a two-stage
approach. A purified graph is firstly obtained by ``pre-processing''
the input graph structure leveraging on information from the node
feature. Next, a GNN model is trained based on this purified graph.
For example, in the GNN-Jaccard method \citep{wu2019adversarial},
a new graph is obtained by removing the edges with small ``Jaccard
similarity.'' In \citet{entezari2020all}, observing that adversarial
attacks can scale up the rank of the graph adjacency matrix, the authors
propose to use a low-rank approximation version of the given graph
adjacency matrix as a substitute. In the second category, the graph
adjacency matrix in a GNN model is treated as an unknown, a purified
graph structure with a parameterized form will be ``learned'' through
optimizing the supervised GNN training loss \citep{zhu2021deep}.
For example, in \citet{franceschi2019learning}, the graph adjacency
matrix is directly learned with a GNN in a bilevel optimization way,
where a full parametrization of the graph adjacency matrix is adopted.
Moreover, under this full parametrization setting, structural regularizers
are adopted in \citet{jin2020graph,luo2021learning} as augmentations
on the training loss function to promote certain properties of the
purified graph. Besides the full parametrization approach, a multi-head
weighted cosine similarity metric function \citep{chen2020iterative}
and a GNN model \citep{yu2020graph} have also been used to parameterize
the graph adjacency matrix for structure learning.

Going beyond purifying the graph structures to robustify the GNN models,
there are also efforts on designing robust GNN architectures via directly
designing the feature aggregation schemes. Under the observation that
aggregation functions such as sum, weighted mean, or the max operations
can be arbitrarily distorted by only a single outlier node, \citet{geisler2020reliable,wang2020provably,zhang2020feature}
try to design robust GNN models via designing robust aggregation functions.
Moreover, some works apply the attention mechanism \citep{velivckovic2018graph}
to mitigate the influence of adversarial perturbations. For example,
\citet{zhu2019robust} consider the node feature following a Gaussian
distribution and use the variance information to determine the attention
scores. \citet{tang2020transferring} use clean graph information
and their adversarial counterparts to train an attention mechanism
to learn to assign small attention scores to the perturbed edges.
In \citet{zhang2020gnnguard}, the authors define an attention mechanism
based on the similarity of neighboring nodes. 

Different from existing approaches to robustify GNNs, in this work,
we propose a novel robust and interpretable message passing scheme
with adaptive structure (ASMP). Based on ASMP, a family of GNN models
with adaptive structure (ASGNN) can be designed. Prior works have
revealed that the message passing processes in a class of GNNs are
actually (unrolled) gradient steps for solving a graph signal denoising
(GSD) problem \citep{zhu2021interpreting,ma2021unified,zhang2022towards}.
ASMP is actually generated by an alternating (proximal) gradient descent
algorithm for simultaneously denoising the graph signal and the graph
structure. Designed in such a principled way, ASMP is not only friendly
to back-propagation training but also achieves the desired structure
adaptivity with a theoretical convergence guarantee. Once trained,
ASMP can be naturally interpreted as a parameter-optimized iterative
algorithm. This work falls into the category of GNN architecture designs.
Conceptually different from the existing robustified GNNs with \textit{fixed}
graph structure, ASGNN interweaves the graph purification process
and the message passing process, which makes it possible to conduct
message passing over different graph structures at different layers,
i.e., in an \textit{adaptive} graph structure fashion. Thus, an edge
might be excluded in some layers but included in other layers, depending
on the dynamic structure learning process. Such property allows more
fine-grained handling of perturbations than existing graph purification
methods that use a single graph in the entire GNN. To be more specific,
the major contributions of this work are highlighted in the following.
\begin{itemize}
\item We propose a novel message passing scheme over graphs called ASMP
with convergence guarantee and specifications. To the best of our
knowledge, ASMP is the first message passing scheme with adaptive
structure that is designed based on an optimization problem. 
\item Based on ASMP, a family of GNN models with adaptive structure, named
ASGNN, are further introduced. The adaptive structure in ASGNN allows
more fine-grained handling of noisy graph structures and strengthens
the model robustness against adversarial attacks. 
\item Extensive experiments under various adversarial attack scenarios showcase
the superiority of the proposed ASGNN. The numerical results corroborate
that the adaptive structure property inherited in ASGNN can help mitigate
the impact of perturbed graph structure. 
\end{itemize}

\section{Preliminaries and Background}

An unweighted graph with self-loops is denoted as $\mathcal{G}=(\mathcal{V},\mathcal{E})$,
where $\mathcal{V}$ and $\mathcal{E}$ denote the node set and the
edge set, respectively. The graph adjacency matrix is given by $\mathbf{A}\in\mathbb{R}^{N\times N}$.
We denote by $\mathbf{1}$ and $\mathbf{I}$ the all-one column vector
and the identity matrix, respectively. Given $\mathbf{D}=\mathrm{Diag}\left(\mathbf{A}\mathbf{1}\right)\in\mathbb{R}^{N\times N}$
as the diagonal degree matrix, the Laplacian matrix is defined as
$\mathbf{L}=\mathbf{D}-\mathbf{A}$. We denote by $\mathbf{A}_{{\rm rw}}=\mathbf{D}^{-1}\mathbf{A}$
the random walk (or row-wise) normalized adjacency matrix and by $\mathbf{A}_{{\rm sym}}=\mathbf{D}^{-\frac{1}{2}}\mathbf{A}\mathbf{D}^{-\frac{1}{2}}$
the symmetric normalized adjacency matrix. Subsequently, the random
walk normalized and symmetric normalized Laplacian matrices are defined
as $\mathbf{L}_{{\rm rw}}=\mathbf{I}-\mathbf{D}^{-1}\mathbf{A}$ and
$\mathbf{L}_{{\rm sym}}=\mathbf{I}-\mathbf{D}^{-\frac{1}{2}}\mathbf{A}\mathbf{D}^{-\frac{1}{2}}$,
respectively. $\mathbf{X}\in\mathbb{R}^{N\times M}$ ($M$ is assumed
to be the dimension of the node feature) is a node feature matrix
or a graph signal, and its $i$-th row $\mathbf{X}_{i,:}$ represents
the feature vector at the $i$-th node with $i=1,\ldots,N$. $\mathbf{X}_{ij}$
(or $\left[\mathbf{X}\right]_{ij}$) denotes the $(i,j)$-th element
of $\mathbf{X}$ with $i,j=1,\ldots,N$. For vector $\mathbf{X}_{i,:}$,
$\mathbf{X}_{i,:}^{-1}$ represents its element-wise inverse.

\subsection{GNNs as Graph Signal Denoising \label{subsec:GNNs-as-GSD}}

In the literature \citep{yang2021graph,panunified,zhu2021interpreting},
it has been realized that the message passing layers for feature learning
in many GNN models could be uniformly interpreted as gradient steps
for minimizing certain energy functions, which carries a meaning of
GSD \citep{ma2021unified}. Recently, \citet{zhang2022towards} further
showed that some popular GNNs are neural networks induced from unrolling
(proximal) gradient descent algorithms for solving specific GSD problems.
Taking the approximate personalized propagation of neural predictions
(APPNP) model \citep{klicpera2019combining} as an example, the initial
node feature matrix $\mathbf{Z}$ is first pre-propcessed by a multilayer
perceptron $g_{\theta}(\cdot)$ with model parameter $\boldsymbol{\theta}$
producing an output $\mathbf{X}=g_{\theta}(\mathbf{Z})$, and then
$\mathbf{X}$ is fed into a $K$-layer message passing scheme given
as follows: 
\begin{equation}
\mathbf{H}^{(0)}=\mathbf{X},\ \ \ \ \mathbf{H}^{(k+1)}=\left(1-\alpha\right)\mathbf{A}_{{\rm sym}}\mathbf{H}^{(k)}+\alpha\mathbf{X},\ \ \text{for}\ k=0,\ldots,K-1,\label{eq:APPNP}
\end{equation}
where $\mathbf{H}^{(0)}$ denotes the input feature of the message
passing process, $\mathbf{H}^{(k)}$ represents the learned feature
after the $k$-th layer, and $\alpha$ is the teleport probability.
Therefore, the message passing of an APPNP model is fully specified
by two parameters, namely, a graph structure matrix $\mathbf{A}_{{\rm sym}}$
and a parameter $\alpha$, in which $\mathbf{A}_{{\rm sym}}$ assumes
to be known beforehand and $\alpha$ is treated as a hyperparameter.

From an optimization perspective, the message passing process in Eq.
\eqref{eq:APPNP} can be seen as executing $K$ steps of gradient
descent to solve a GSD problem with initialization $\mathbf{H}^{(0)}=\mathbf{X}$
and step size $0.5$ \citep{zhu2021interpreting,ma2021unified,zhang2022towards},
which is given by
\begin{equation}
\begin{aligned} & \underset{\mathbf{H}\in\mathbb{R}^{N\times M}}{\mathsf{minimize}} &  & \alpha\left\Vert \mathbf{H}-\mathbf{X}\right\Vert _{{\rm F}}^{2}+\left(1-\alpha\right)\mathrm{Tr}\bigl(\mathbf{H}^{\top}\mathbf{L}_{{\rm sym}}\mathbf{H}\bigr),\end{aligned}
\label{eq:GSD}
\end{equation}
where $\mathbf{X}$ and $\alpha$ are given and share the same meaning
as in Eq. \eqref{eq:APPNP}. In Problem \eqref{eq:GSD}, the first
term is a fidelity term forcing the recovered graph signal $\mathbf{H}$
to be as close as possible to a noisy graph signal $\mathbf{X}$,
and the second term is the symmetric normalized Laplacian smoothing
term measuring the variation of the graph signal $\mathbf{H}$, which
can be explicitly expressed as 
\begin{equation}
\mathrm{Tr}\left(\mathbf{H}^{\top}\mathbf{L}_{{\rm sym}}\mathbf{H}\right)=\frac{1}{2}\sum_{i=1}^{N}\sum_{j=1}^{N}\mathbf{A}_{ij}\left\Vert \frac{\mathbf{H}_{i,:}}{\sqrt{\mathbf{D}_{ii}}}-\frac{\mathbf{H}_{j,:}}{\sqrt{\mathbf{D}_{jj}}}\right\Vert _{2}^{2}.\label{eq:Laplacian}
\end{equation}
For more technical discussions on relationships between GNNs with
iterative optimization algorithms for solving GSD problems, please
refer to \citet{ma2021unified,zhang2022towards}. Apart from using
the lens of optimization to interpret existing GNN models, there are
also literature \citep{liu2021elastic,chen2021graph,fu2022p} working
on building new GNN architectures based on designing novel optimization
problems and the corresponding iterative algorithms (more discussions
are provided in Appendix \ref{Appendix:Related-Work}).

\subsection{Graph Learning with Structural Regularizers \label{subsec:GSL-Regularizer}}

Structural regularizers are commonly adopted to promote certain desirable
properties when learning a graph \citep{kalofolias2016learn,pu2021learning}.
In the following, we discuss several widely used graph structural
regularizers which will be incorporated into the design of ASMP. We
denote the learnable graph adjacency matrix as $\mathbf{S}$ satisfying
$\mathbf{S}\in\mathcal{S}$, where 
\[
\mathcal{S}=\left\{ \mathbf{S}\in\mathbb{R}^{N\times N}\mid0\leq\mathbf{S}_{ij}\leq1,\ \text{for}\ i,j=1,\ldots,N\right\} 
\]
defines the class of adjacency matrices. Under the assumption that
node feature changes smoothly between adjacent nodes \citep{ortega2018graph},
the Laplacian smoothing regularization term is commonly considered
in graph structure learning. Eq. \eqref{eq:Laplacian} is the symmetric
normalized Laplacian smoothing term, and a random walk normalized
alternative can be similarly defined by replacing $\mathbf{L}_{{\rm sym}}$
in Eq. \eqref{eq:Laplacian} by $\mathbf{L}_{{\rm rw}}$.

Real-world graphs are normally sparsely connected, which can be represented
by sparse adjacency matrices. Moreover, it is also observed that singular
values of these adjacency matrices are commonly small \citep{zhou2013learning,kumar2020unified}.
However, a noisy adjacency matrix (e.g., one perturbed by adversarial
attacks) tends to be dense and to gain singular values in larger magnitudes
\citep{jin2020graph}. In view of this, graph structural regularizers
for promoting sparsity and/or suppressing the singular values are
widely adopted in the literature of graph learning \citep{kalofolias2016learn,egilmez2017graph,dong2019learning}.
Specifically, the $\ell_{1}$-norm of the adjacency matrix is often
used to promote sparsity, defined as $\left\Vert \mathbf{S}\right\Vert _{1}=\sum_{i,j=1}^{N}\left|\mathbf{S}_{ij}\right|.$
For penalizing the singular values, the $\ell_{1}$-norm and the $\ell_{2}$-norm
on the singular value vector of the adjacency matrix $\mathbf{S}$
can help. Equivalently, they can be translated to be the nuclear norm
and the Frobenius norm on $\mathbf{S}$, which are given by $\left\Vert \mathbf{S}\right\Vert _{*}=\sum_{i=1}^{N}\sigma_{i}\left(\mathbf{S}\right)$
and $\left\Vert \mathbf{S}\right\Vert _{{\rm F}}=\sqrt{\sum_{i=1}^{N}\sigma_{i}^{2}\left(\mathbf{S}\right)}$,
respectively, where $\sigma_{1}\left(\mathbf{S}\right)\geq\cdots\geq\sigma_{N}\left(\mathbf{S}\right)$
denote the ordered singular values of $\mathbf{S}$. These two regularizers
both restrict the scale of the singular values while the nuclear norm
also promotes low-rankness. A recent study \citep{deng2022garnet}
points out that graph learning methods with low-rank promoting regularizers
may lose a wide range of spectrum of the clean graph corresponding
to important structure in the spatial domain. Thus, the nuclear norm
regularizer may impair the quality of the reconstructed graph and
therefore limit the performance of GNNs. Besides, the nuclear norm
is not amicable for back-propagation and incurs high computational
complexity \citep{luo2021learning}. Arguably, the Frobenius norm
of $\mathbf{S}$ is a more suitable regularizer for graph structure
learning in comparison with the nuclear norm.

\section{The Proposed Graph Neural Networks}

In this section, we first motivate the design principle based on jointly
node feature learning and graph structure learning. Then, we develop
an efficient optimization algorithm for solving this optimization
problem, which eventually leads to a novel message passing scheme
with adaptive structure (ASMP). After that, we provide interpretations,
convergence guarantees, and specifications of ASMP. Finally, integrating
ASMP with deep neural networks ends up with a new family of GNNs with
adaptive structure, named ASGNNs.

\subsection{A Novel Design Principle with Adaptive Graph Structure}

As discussed in Section \ref{subsec:GNNs-as-GSD}, the message passing
procedure in many popular GNNs can be viewed as performing graph signal
denoising (or node feature learning) \citep{zhu2021interpreting,ma2021unified,panunified,zhang2022towards}
over a prefixed graph. 
Unfortunately, if some edges in the graph are task-irrelevant or even
maliciously manipulated, the node feature learned may not be appropriate
for the downstream tasks. Motivated by this, we propose a new design
principle for message passing, that is, to learn the node feature
and the graph structure simultaneously. It enables learning an adaptive
graph structure from the feature for the message passing procedure.
Hence, such a message passing scheme can potentially improve robustness
against noisy input graph structure.

Specifically, we construct an optimization objective by augmenting
the GSD objective in Eq. \eqref{eq:GSD} (we have used a random walk
normalized graph Laplacian smoothing term) with a structural fidelity
term $\|\mathbf{S}-\mathbf{A}\|_{{\rm F}}^{2}$, where $\mathbf{A}$
is the given initial graph adjacency matrix, and the structural regularizers
$\left\Vert \mathbf{S}\right\Vert _{1}$ and $\left\Vert \mathbf{S}\right\Vert _{{\rm F}}^{2}$.
Then we obtain the following optimization problem:\begin{equation} \underset{\mathbf{H}\in\mathbb{R}^{N\times M},\;\mathbf{S}\in\mathcal{S}}{\mathsf{minimize}}\ \ \ p\left(\mathbf{H},\mathbf{S}\right)=\rlap{$\overbrace{\phantom{\|\mathbf{H}-\mathbf{X}\|_{\rm F}^{2}+\lambda\mathrm{Tr}\left(\mathbf{H}^{\top}\mathbf{L}_{\rm rw}\mathbf{H}\right)}}^{\text{feature learning}}$}\|\mathbf{H}-\mathbf{X}\|_{\rm F}^{2}+\underbrace{\lambda\mathrm{Tr}\left(\mathbf{H}^{\top}\mathbf{L}_{\rm rw}\mathbf{H}\right)+\gamma\|\mathbf{S}-\mathbf{A}\|_{\rm F}^{2}+\mu_1\left\Vert \mathbf{S}\right\Vert _{1}+\mu_2\left\Vert \mathbf{S}\right\Vert _{\rm F}^{2}}_{\text{structure learning}},\label{eq:Structure Signal Denoising} \end{equation}

where $\mathbf{H}$ is the feature variable, $\mathbf{S}$ is the
structure variable, and $\gamma$, $\lambda$, $\mu_{1}$, and $\mu_{2}$
are parameters balancing different terms. To enable the interplay
between feature learning and structure learning, the Laplacian smoothing
term is concerned with $\mathbf{S}$ rather than $\mathbf{A}$, i.e.,
$\mathbf{L}_{{\rm rw}}=\mathbf{I}-\mathbf{D}^{-1}\mathbf{S}$ with
$\mathbf{D}=\mathrm{Diag}\left(\mathbf{S}\mathbf{1}\right)$. When
adversarial attacks exist, a perturbed adjacency matrix $\mathbf{A}$
will be generated. Since attacks are generally designed to be unnoticeable
\citep{jin2021adversarial}, the perturbed graph adjacency matrix
is largely similar to the original graph matrix in value. In view
of this, we also include a structural fidelity term $\|\mathbf{S}-\mathbf{A}\|_{{\rm F}}^{2}$.
The motivation for introducing the last two regularizers has been
elaborated in Section \ref{subsec:GSL-Regularizer}.

\subsection{ASMP: Message Passing with Adaptive Structure}

\begin{figure}
\centering{}\includegraphics[width=1\columnwidth]{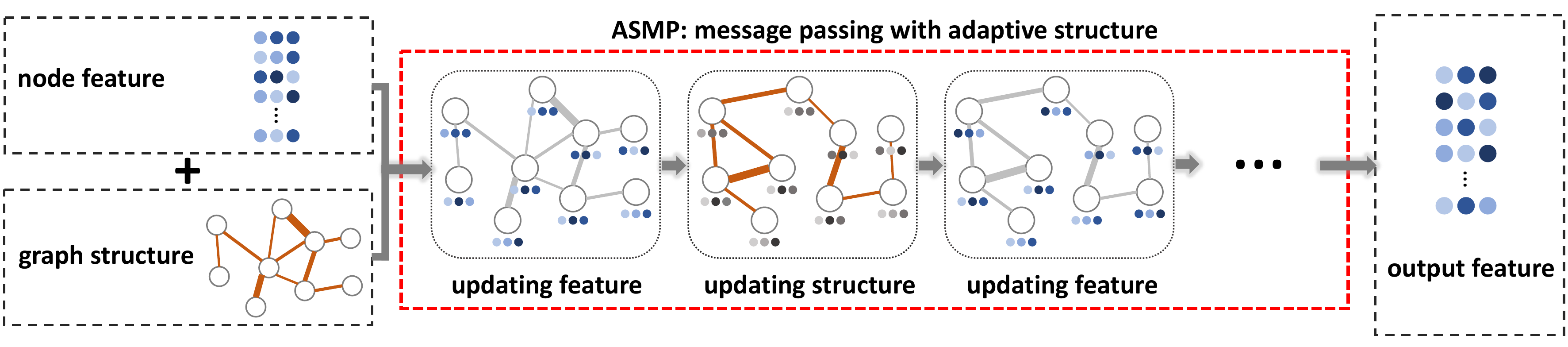}\caption{Illustration of ASMP. (ASMP takes the node feature matrix $\mathbf{X}$
and graph structure matrix $\mathbf{A}$ as input. Different colors
of the feature indicate different embedding values. Different width
of the edges indicate different weight values. ASMP updates the node
feature and the graph structure in an alternating way.) \label{fig:ASMP}}
\end{figure}

Following the idea that the message passing of a GNN model can be
derived based on the optimization of a GSD objective function \citep{ma2021unified,zhang2022towards},
we can obtain a message passing scheme from Problem \eqref{eq:Structure Signal Denoising}.
Different from the existing GSD problems for GNN model design with
only the feature variable, Problem \eqref{eq:Structure Signal Denoising}
is nonconvex and much more challenging. To obtain an efficient iterative
algorithm that is friendly to back-propagation training, we propose
to use the alternating (proximal) gradient descent method \citep{parikh2014proximal},
i.e., alternatingly optimizing one variable by taking one (proximal)
gradient step at a time with the other variable fixed. (Note that
a joint optimization approach is also eligible, while it would lead
to slower convergence than the alternating optimization approach.
More details can be found in Appendix \ref{sec:Discussion-on-joint}.)

We denote by $\mathbf{H}^{(k)}$ and $\mathbf{S}^{(k)}$ the variables
at the $k$-th iteration ($k=0,\ldots,K$). In the following, the
update rules for $\mathbf{H}$ and $\mathbf{S}$ will be discussed,
respectively.

\textbf{Updating node feature matrix $\mathbf{H}$}: Given $\{\mathbf{H}^{(k)},\mathbf{S}^{(k)}\}$,
the subproblem with respect to feature matrix $\mathbf{H}$ is given
by 
\begin{equation}
\underset{\mathbf{H}\in\mathbb{R}^{N\times M}}{\mathsf{minimize}}\ \ \ \|\mathbf{H}-\mathbf{X}\|_{{\rm F}}^{2}+\lambda\mathrm{Tr}\Bigl(\mathbf{H}^{\top}\mathbf{L}_{{\rm rw}}^{(k)}\mathbf{H}\Bigr),\label{eq:H-objective}
\end{equation}
where $\mathbf{L}_{{\rm rw}}^{(k)}=\mathbf{I}-\mathrm{Diag}\bigl(\mathbf{S}^{(k)}\mathbf{1}\bigr)^{-1}\mathbf{S}^{(k)}$.
One gradient step for $\mathbf{H}$ is computed as 
\begin{align*}
\mathbf{H}^{(k+1)} & =\mathbf{H}^{(k)}-\eta_{1}\Bigl(2\mathbf{H}^{(k)}-2\mathbf{X}+2\lambda\mathbf{L}_{{\rm rw}}^{(k)}\mathbf{H}^{(k)}\Bigr)\\
 & =\mathbf{H}^{(k)}-\eta_{1}\left(2\mathbf{H}^{(k)}-2\mathbf{X}+2\lambda\left(\mathbf{I}-\mathrm{Diag}\bigl(\mathbf{S}^{(k)}\mathbf{1}\bigr)^{-1}\mathbf{S}^{(k)}\right)\mathbf{H}^{(k)}\right)\\
 & =\left(1-2\eta_{1}-2\eta_{1}\lambda\right)\mathbf{H}^{(k)}+2\eta_{1}\lambda\mathrm{Diag}\bigl(\mathbf{S}^{(k)}\mathbf{1}\bigr)^{-1}\mathbf{S}^{(k)}\mathbf{H}^{(k)}+2\eta_{1}\mathbf{X},
\end{align*}
where $\eta_{1}$ denotes the step size.

\textbf{Updating graph structure matrix $\mathbf{S}$}: Given $\{\mathbf{H}^{(k+1)},\mathbf{S}^{(k)}\}$
and $\mathrm{Tr}(\mathbf{H}^{(k+1)\top}\mathbf{L}_{{\rm rw}}\mathbf{H}^{(k+1)})=\mathrm{Tr}(\mathbf{H}^{(k+1)\top}\mathbf{H}^{(k+1)})-\mathrm{Tr}(\mathbf{H}^{(k+1)\top}\mathrm{Diag}(\mathbf{S}\mathbf{1})^{-1}\mathbf{S}\mathbf{H}^{(k+1)})$,
the subproblem for $\mathbf{S}$ becomes 
\begin{equation}
\underset{\mathbf{S}\in\mathcal{S}}{\mathsf{minimize}}\ \ \ \gamma\|\mathbf{S}-\mathbf{A}\|_{{\rm F}}^{2}-\lambda\mathrm{Tr}\left(\mathbf{H}^{(k+1)\top}\mathrm{Diag}\bigl(\mathbf{S}\mathbf{1}\bigr)^{-1}\mathbf{S}\mathbf{H}^{(k+1)}\right)+\mu_{1}\left\Vert \mathbf{S}\right\Vert _{1}+\mu_{2}\left\Vert \mathbf{S}\right\Vert _{{\rm F}}^{2}.\label{eq:S-objective}
\end{equation}
Due to the non-smoothness of the objective function, we apply one
step of the proximal gradient descent \citep{parikh2014proximal}
for this problem. Define 
\[
\begin{aligned}\mathbf{T}^{(k)} & =\left(2\gamma+2\mu_{2}\right)\mathbf{S}^{(k)}-2\gamma\mathbf{A}-\lambda\mathrm{Diag}\bigl(\mathbf{S}^{(k)}\mathbf{1}\bigr)^{-1}\mathbf{H}^{(k+1)}(\mathbf{H}^{(k+1)})^{\top}\\
 & \hspace{2.75cm}+\lambda\mathrm{Diag}\left(\mathrm{Diag}\bigl(\mathbf{S}^{(k)}\mathbf{1}\bigr)^{-1}\mathbf{S}^{(k)}\mathbf{H}^{(k+1)}(\mathbf{H}^{(k+1)})^{\top}\mathrm{Diag}\bigl(\mathbf{S}^{(k)}\mathbf{1}\bigr)^{-1}\right)\mathbf{1}^{\top}.
\end{aligned}
\]
One step of proximal gradient descent is given as follows (details
are given in Appendix \ref{Appendix: Derivation of ProxGD in ASMP}):
\begin{equation}
\mathbf{S}^{(k+1)}=\mathrm{prox}_{\eta_{2}\left(\mu_{1}\left\Vert \cdot\right\Vert _{1}+\mathbb{I}_{\mathcal{S}}\left(\cdot\right)\right)}\bigl(\mathbf{S}^{(k)}-\eta_{2}\mathbf{T}^{(k)}\bigr),\label{ASMP: ProxGD Step}
\end{equation}
where $\eta_{2}$ is the step size and $\mathbb{I}_{\mathcal{S}}\left(\mathbf{S}\right)$
denotes the indicator function taking value 0 if $\mathbf{S}\in\mathcal{S}$
and $+\infty$ otherwise. Moreover, the proximal operator in Eq. \eqref{ASMP: ProxGD Step}
can be computed analytically as 
\[
\mathbf{S}^{(k+1)}=\mathrm{min}\Bigl\{1,\mathrm{ReLU}\bigl(\mathbf{S}^{(k)}-\eta_{2}\mathbf{T}^{(k)}-\eta_{2}\mu_{1}\mathbf{1}\mathbf{1}^{\top}\bigr)\Bigr\},
\]
where $\mathrm{ReLU}(\mathbf{X})=\max\{\mathbf{0},\mathbf{X}\}$.
In conclusion, the overall procedure of ASMP can be summarized as
follows: 
\begin{equation}
\begin{cases}
\mathbf{H}^{(k+1)}=\left(1-2\eta_{1}-2\eta_{1}\lambda\right)\mathbf{H}^{(k)}+2\eta_{1}\lambda\mathrm{Diag}\bigl(\mathbf{S}^{(k)}\mathbf{1}\bigr)^{-1}\mathbf{S}^{(k)}\mathbf{H}^{(k)}+2\eta_{1}\mathbf{X},\\
\mathbf{S}^{(k+1)}=\mathrm{min}\left\{ 1,\mathrm{ReLU}\left(\mathbf{S}^{(k)}-\eta_{2}\mathbf{T}^{(k)}-\eta_{2}\mu_{1}\mathbf{1}\mathbf{1}^{\top}\right)\right\} ,
\end{cases}k=0,\ldots,K-1.\tag{ASMP}\label{ASMP: Procedure}
\end{equation}
The ASMP can be interpreted as the standard message passing (i.e.,
the update step of $\mathbf{H}$) with extra operations that adaptively
adjust the graph structure (i.e., the update step of $\mathbf{S}$).
Therefore, an edge of the graph included in some layers may be excluded
or down-weighted in other layers. A pictorial illustration of the
ASMP procedure is provided in Figure \ref{fig:ASMP}. A $K$-layer
ASMP can be fully specified by parameters $\gamma$, $\lambda$, $\mu_{1}$,
$\mu_{2}$, $\eta_{1}$, and $\eta_{2}$, which we generally denote
as $\mathrm{ASMP}_{K}\bigl(\mathbf{X},\mathbf{A},\gamma,\lambda,\mu_{1},\mu_{2},\eta_{1},\eta_{2}\bigr)$.

Note that ASMP is general enough to cover several existing propagation
rules as special cases. 
\begin{rem}[Special cases]
\label{rem:special_case} If we use a fixed graph structure $\mathbf{S}^{(0)}=\cdots=\mathbf{S}^{(K)}=\mathbf{A}$
in ASMP, i.e., $\mu_{1}=\mu_{2}=\gamma=0$, the ASMP reduces to a
classical message passing procedure that only performs feature learning.
Specifically, with $\eta_{1}=\frac{1}{2+2\lambda}$ and the symmetric
normalized adjacency matrix, ASMP can be written as 
\begin{equation}
\mathbf{H}^{(k+1)}=\frac{\lambda}{1+\lambda}\mathbf{A}_{{\rm sym}}\mathbf{H}^{(k)}+\frac{1}{1+\lambda}\mathbf{X}.\label{eq:reduced-ASMP}
\end{equation}
Case I: when $\lambda=\frac{1}{\alpha}-1$, the operation in Eq. \eqref{eq:reduced-ASMP}
becomes the message passing rule of APPNP \citep{klicpera2019combining}:
\[
\mathbf{H}^{(k+1)}=\left(1-\alpha\right)\mathbf{A}_{{\rm sym}}\mathbf{H}^{(k)}+\alpha\mathbf{X}.
\]
Case II: when $\lambda=\infty$, the operation in Eq. \eqref{eq:reduced-ASMP}
becomes the simple aggregation in many GNN models such as the GCN
model \citep{kipf2017semi} and the simple graph convolution (SGC)
model \citep{wu2019simplifying}: 
\[
\mathbf{H}^{(k+1)}=\mathbf{A}_{{\rm sym}}\mathbf{H}^{(k)}.
\]
\end{rem}

Instead of updating both $\mathbf{S}$ and $\mathbf{H}$ once, we
can also choose to update them for several steps. The convergence
of ASMP is guaranteed with proper selections of the step sizes as
demonstrated in Theorem \ref{ASMP: Convergence Theorem}. Before proceeding
to the convergence result, we first introduce some standard assumptions
on the node feature vectors and the degree matrices, which are widely
adopted in the literature \citep{garg2020generalization,liao2021pac,cong2021provable}. 
\begin{assumption}
\label{assu:feature norm}The energy of the node feature is uniformly
upperbounded, i.e., $||\mathbf{H}_{i,:}^{(k)}||_{2}\leq B$ for $i=1,\ldots,N$
and $k=0,\ldots,K$.
\end{assumption}

\begin{assumption}
\label{assu:degree matrix}The diagonal elements of the degree matrix
is lowerbounded by a positive constant, i.e., $\min_{i}\mathbf{D}_{ii}=c>0$
for $i=1,\ldots,N$.
\end{assumption}

\begin{thm}
\label{ASMP: Convergence Theorem} Let $\mathbf{H}^{(0)}=\mathbf{X}$
and $\mathbf{S}^{(0)}=\mathbf{A}$. Under Assumption \ref{assu:feature norm}
and Assumption \ref{assu:degree matrix}, the sequence $\{\mathbf{H}^{(k)},\mathbf{S}^{(k)}\}_{k=1}^{K}$
generated by \eqref{ASMP: Procedure} with $0<\eta_{1}<\frac{1}{1+2\lambda}$
and $0<\eta_{2}<\frac{1}{\gamma+\mu_{2}+(1+\frac{1}{c}N\sqrt{N})\frac{\lambda}{c^{2}}N^{2}B^{2}}$
converges to a first-order stationary point of Problem \eqref{eq:Structure Signal Denoising}
denoted as $\{\mathbf{H}^{\ast},\mathbf{S}^{\ast}\}$ with rate 
\[
\inf_{k\geq K}\left\{ \bigl\Vert\mathbf{H}^{(k+1)}-\mathbf{H}^{(k)}\bigr\Vert_{{\rm F}}^{2}+\bigl\Vert\mathbf{S}^{(k+1)}-\mathbf{S}^{(k)}\bigr\Vert_{{\rm F}}^{2}\right\} \leq\frac{1}{\rho K}\left(p\bigl(\mathbf{H}^{(0)},\mathbf{S}^{(0)}\bigr)-p\bigl(\mathbf{H}^{\ast},\mathbf{S}^{\ast}\bigr)\right),
\]
where $\rho$ is a constant depending on the step sizes and Lipschitz
constants, and $p\left(\mathbf{H},\mathbf{S}\right)$ represents the
objective of Problem \eqref{eq:Structure Signal Denoising}.
\end{thm}

\begin{proof}
The proof for Theorem \ref{ASMP: Convergence Theorem} is in Appendix
\ref{Appendix: Convergence Proof}. Note that if multiple updating
steps are used for $\mathbf{S}$ and $\mathbf{H}$ in ASMP, this convergence
result still holds \citep{bolte2014proximal,nikolova2017alternating}. 
\end{proof}
\textcolor{red}{}

\subsection{ASGNN: Graph Neural Networks with Adaptive Structure}

In this section, we introduce a family of GNNs leveraging the ASMP
scheme. Integrating \eqref{ASMP: Procedure} with a machine learning
model $g_{\theta}(\cdot)$ (e.g., a multilayer perceptron) with $\mathbf{H}^{(0)}=\mathbf{X}=g_{\theta}(\mathbf{Z})$,
a $K$-layer ASGNN model is defined as follows: 
\[
\mathbf{H}^{(K)}=\mathrm{ASMP}_{K}\Bigl(\mathbf{H}^{(0)},\mathbf{A},\gamma,\lambda,\mu_{1},\mu_{2},\eta_{1},\eta_{2}\Bigr).
\]
In ASGNN, we have chosen a decoupled architecture similar to APPNP
\citep{klicpera2019combining} and deep adaptive GNN (DAGNN) \citep{liu2021graph}.
Specially, in ASGNN, the model $g_{\theta}$ will first transform
the initial node feature as $\mathbf{X}=g_{\theta}(\mathbf{Z})$,
and then ASMP performs $K$ steps of message passing with input $g_{\theta}(\mathbf{Z})$.

The hyper-parameters in ASMP, i.e., $\gamma,$ $\lambda,$ $\mu_{1},$
and $\mu_{2}$, are set to be weights to be learned from the downstream
tasks. (It should be noted that since the parameter $\lambda$ can
be either positive or negative, ASMP is capable of handling both homophily
and heterophily graphs.) For example, in semi-supervised node classification
tasks, the loss function is chosen as the cross-entropy classification
loss on the labeled nodes and the whole model is trained in an end-to-end
way. Since ASMP is derived from the alternating (proximal) gradient
descent algorithm, a trained ASMP is naturally a parameter-optimized
iterative algorithm. The step sizes $\eta_{1}$ and $\eta_{2}$ in
ASMP can be chosen according to the results in Theorem \ref{ASMP: Convergence Theorem}.
However, such choices seem to be too conservative in practice and
may lead to slow convergence. Thus, we may also consider the step
sizes $\eta_{1}$ and $\eta_{2}$ as learnable parameters. Convergence
property of ASMP with learned step sizes will be showcased in the
experiments. In conclusion, there are in total six parameters in ASMP
considered during the learning process.

In this paper, we have focused on problems in which there is an initial
graph structure, while the use of ASGNN may also be extended to scenarios
where the initial structure is not available. In such case, we can
first create a $k$-nearest neighbor graph or use some optimization
methods \citep{dong2016learning,kalofolias2016learn,kumar2020unified}
to learn a graph structure based on the node feature. Such extensions
of ASGNN can be promising future research directions.\textcolor{red}{}

\section{Experiments}

In this section, we conduct experiments to validate the effectiveness
of the proposed ASGNN model. First, we introduce the experimental
settings. Then, we assess the performance of ASGNN on semi-supervised
node classifications tasks and investigate the benefits of introducing
adaptive structure into GNNs against global attacks and targeted attacks.
Finally, we analyze the structure denoising ability and the convergence
property of ASMP with the learned step sizes.

\subsection{Experiment Settings}

\textbf{Datasets}: We perform numerical experiments on 4 real-world
citation graphs, i.e., Cora, Citeseer \citep{sen2008collective},
Cora-ML \citep{bojchevski2018deep}, and ACM \citep{wang2019heterogeneous},
and only consider the largest connected component in each dataset.

\textbf{Baselines}: To evaluate the effectiveness of ASGNN, we compare
it with GCN and several benchmarks that are designed from different
perspectives to robustify the GNNs, including GCN-Jaccard \citep{wu2019adversarial}
that pre-processes the graph by eliminating edges with low Jaccard
similarity of node feature vectors, GCN-SVD \citep{entezari2020all}
that applies the low-rank approximation of the given graph adjacency
matrix, Pro-GNN \citep{jin2020graph} that jointly learns a graph
structure and a GNN model guided by some predefined structural priors,
and Elastic GNN \citep{liu2021elastic} that utilizes trend filtering
instead of Laplacian smoothing to promote robustness. The code is
implemented based on PyTorch Geometric \citep{Fey/Lenssen/2019}.
For GCN-Jaccard, GCN-SVD, and Pro-GNN, we use the implementation provided
in DeepRobust \citep{li2020deeprobust}. For Elastic GNN, we follow
the implementation provided in the original paper \citep{liu2021elastic}.

\textbf{Parameter settings}: For all the experimental results, we
give the average performance and standard variance with 10 independent
trials. For each graph, we randomly select 10\%/10\%/80\% of nodes
for training, validation, and testing. The Adam optimizer is used
in all experiments. The models' hyperparameters are tuned based on
the results of the validation set. The search space of hyperparameters
are as follows: 1) learning rate: \{0.005, 0.01, 0.05\}; 2) weight
decay: \{0, 5e-5, 5e-4\}; 3) dropout rate: \{0.1, 0.5, 0.8\}; 4) model
depth: \{2, 4, 8, 16\}. For GCN-Jaccard, the threshold of Jaccard
similarity for removing dissimilar edges is chosen from \{0.01, 0.02,
0.03, 0.04, 0.05, 0.1\}. For GCN-SVD, the reduced rank of the graph
is tuned from \{5, 10, 15, 50, 100, 200\}. For Elastic GNN, the regularization
coefficients are chosen from \{3, 6, 9\}. For Pro-GNN, we adopt the
hyperparameters provided in their paper \citep{jin2020graph}.

\subsection{Performance Under Adversarial Attack}

The performance of the compared models is evaluated under the training-time
adversarial attacks \citep{wang2019attacking,zugner2019adversarial},
i.e., the graph is first attacked, and then the GNN models are trained
on the perturbed graph. In the following, we conduct experiments under
both the global attack and the targeted attack. Specifically, the
global attack aims to reduce the overall performance of GNNs \citep{zugner2019adversarial}
while the targeted attack aims to fool GNNs on some specific nodes
\citep{zugner2018adversarial}.

\begin{table}
\begin{centering}
\caption{Node classification performance (accuracy $\pm$ std) under global
attack (\textbf{Bold}: the best model; \uwave{wavy}: the runner-up
model)\label{tab:global}}
\par\end{centering}
\centering{}\resizebox{1 \textwidth}{!}{%
\begin{tabular}{c|c|cccccc}
\hline 
\textbf{Dataset}  & \textbf{Ptb. rate (\%)}  & GCN  & GCN-Jaccard  & GCN-SVD  & Pro-GNN  & Elastic GNN  & ASGNN\tabularnewline
\hline 
\multirow{6}{*}{Cora} & 0  & \uwave{85.34 }{\small{}\uwave{\mbox{$\pm$} 0.39}}  & 81.75 {\small{}{}$\pm$ 0.49}  & 75.15 {\small{}{}$\pm$ 0.64}  & 82.94 {\small{}{}$\pm$ 0.28}  & 84.80 {\small{}{}$\pm$ 0.58}  & \textbf{85.38 }\textbf{\small{}{}$\pm$ 0.24}\tabularnewline
 & 5  & 79.71 {\small{}{}$\pm$ 0.48}  & 77.81 {\small{}{}$\pm$ 0.52}  & 73.71 {\small{}{}$\pm$ 0.42}  & 82.20 {\small{}{}$\pm$ 0.35}  & \uwave{82.26 }{\small{}\uwave{\mbox{$\pm$} 0.69}}  & \textbf{82.31 }\textbf{\small{}{}$\pm$ 0.53}\tabularnewline
 & 10  & 74.28 {\small{}{}$\pm$ 0.79}  & 74.38 {\small{}{}$\pm$ 0.30}  & 65.85 {\small{}{}$\pm$ 0.39}  & 79.30 {\small{}{}$\pm$ 0.64}  & \uwave{79.47 }{\small{}\uwave{\mbox{$\pm$} 1.52}}  & \textbf{80.31 }\textbf{\small{}{}$\pm$ 0.61}\tabularnewline
 & 15  & 69.05 {\small{}{}$\pm$ 0.77}  & 72.54 {\small{}{}$\pm$ 0.31}  & 65.33 {\small{}{}$\pm$ 0.47}  & 77.69 {\small{}{}$\pm$ 0.74}  & \uwave{77.84 }{\small{}\uwave{\mbox{$\pm$} 1.08}}  & \textbf{78.11 }\textbf{\small{}{}$\pm$ 0.76}\tabularnewline
 & 20  & 57.76 {\small{}{}$\pm$ 1.01}  & 71.76 {\small{}{}$\pm$ 0.48}  & 60.85 {\small{}{}$\pm$ 0.74}  & \uwave{74.16 }{\small{}\uwave{\mbox{$\pm$} 1.02}}  & 63.68 {\small{}{}$\pm$ 0.27}  & \textbf{77.04 }\textbf{\small{}{}$\pm$ 0.59}\tabularnewline
 & 25  & 52.67 {\small{}{}$\pm$ 1.00}  & 69.67 {\small{}{}$\pm$ 0.46}  & 59.31 {\small{}{}$\pm$ 0.47}  & \uwave{71.19 }{\small{}\uwave{\mbox{$\pm$} 1.27}}  & 62.90 {\small{}{}$\pm$ 3.37}  & \textbf{75.18 }\textbf{\small{}{}$\pm$ 0.97}\tabularnewline
\hline 
\multirow{6}{*}{Citeseer} & 0  & \uwave{73.97 }{\small{}\uwave{\mbox{$\pm$} 0.54}}  & 72.09 {\small{}{}$\pm$ 0.49}  & 68.34 {\small{}{}$\pm$ 0.39}  & 73.35 {\small{}{}$\pm$ 0.47}  & 73.82 {\small{}{}$\pm$ 0.43}  & \textbf{73.99 }\textbf{\small{}{}$\pm$ 0.93}{\small{}{}}\tabularnewline
 & 5  & 72.57 {\small{}{}$\pm$ 0.93}  & 70.79 {\small{}{}$\pm$ 0.30}  & 67.59 {\small{}{}$\pm$ 0.43}  & 73.16 {\small{}{}$\pm$ 0.42}  & \uwave{73.30 }{\small{}\uwave{\mbox{$\pm$} 0.37}}  & \textbf{73.35 }\textbf{\small{}{}$\pm$ 0.41}\tabularnewline
 & 10  & 71.21 {\small{}{}$\pm$ 1.44}  & 70.27 {\small{}{}$\pm$ 0.62}  & 67.38 {\small{}{}$\pm$ 0.65}  & \uwave{72.78 }{\small{}\uwave{\mbox{$\pm$} 0.79}}  & \uwave{72.78 }{\small{}\uwave{\mbox{$\pm$} 0.66}}  & \textbf{72.83 }\textbf{\small{}{}$\pm$ 0.56}\tabularnewline
 & 15  & 68.00 {\small{}{}$\pm$ 1.04}  & 69.97 {\small{}{}$\pm$ 1.49}  & 66.47 {\small{}{}$\pm$ 0.51}  & 71.55 {\small{}{}$\pm$ 0.73}  & \uwave{71.73 }{\small{}\uwave{\mbox{$\pm$} 1.03}}  & \textbf{71.85 }\textbf{\small{}{}$\pm$ 1.83}\tabularnewline
 & 20  & 59.75 {\small{}{}$\pm$ 0.83}  & 69.49 {\small{}{}$\pm$ 0.71}  & 65.83 {\small{}{}$\pm$ 0.69}  & \uwave{70.07 }{\small{}\uwave{\mbox{$\pm$} 1.12}}  & 61.55 {\small{}{}$\pm$ 1.82}  & \textbf{71.06 }\textbf{\small{}{}$\pm$ 3.09}\tabularnewline
 & 25  & 59.98 {\small{}{}$\pm$ 0.98}  & 68.14 {\small{}{}$\pm$ 0.36}  & 62.34 {\small{}{}$\pm$ 0.61}  & \uwave{69.73 }{\small{}\uwave{\mbox{$\pm$} 0.93}}  & 63.98 {\small{}{}$\pm$ 2.17}  & \textbf{70.03 }\textbf{\small{}{}$\pm$ 3.45}\tabularnewline
\hline 
\multirow{6}{*}{Cora-ML} & 0  & 86.59 {\small{}{}$\pm$ 0.07}  & 84.68 {\small{}{}$\pm$ 0.32}  & 82.96 {\small{}{}$\pm$ 0.27}  & 79.48 {\small{}{}$\pm$ 0.40}  & \textbf{87.01 }\textbf{\small{}{}$\pm$ 0.28}{\small{} } & \uwave{86.68 }{\small{}\uwave{\mbox{$\pm$} 0.43}}\tabularnewline
 & 5  & 80.99 {\small{}{}$\pm$ 0.50}  & 81.80 {\small{}{}$\pm$ 0.37}  & 81.78 {\small{}{}$\pm$ 0.46}  & 78.57 {\small{}{}$\pm$ 0.16}  & \uwave{84.68 }{\small{}\uwave{\mbox{$\pm$} 0.25}}  & \textbf{84.80 }\textbf{\small{}{}$\pm$ 0.80}{\small{}{}}\tabularnewline
 & 10  & 74.57 {\small{}{}$\pm$ 0.75}  & 80.35 {\small{}{}$\pm$ 0.24}  & 81.75 {\small{}{}$\pm$ 0.33}  & 78.74 {\small{}{}$\pm$ 0.84}  & \uwave{82.01 }{\small{}\uwave{\mbox{$\pm$} 0.64}}  & \textbf{83.09 }\textbf{\small{}{}$\pm$ 0.59}{\small{}{}}\tabularnewline
 & 15  & 54.69 {\small{}{}$\pm$ 0.52}  & \textbf{76.53 }\textbf{\small{}{}$\pm$ 0.29}{\small{} } & \uwave{74.7}{\small{}\uwave{6}}\uwave{ }{\small{}\uwave{\mbox{$\pm$}
0.44}}  & 73.62 {\small{}{}$\pm$ 0.85}  & 64.59 {\small{}{}$\pm$ 2.69}  & 73.71 {\small{}{}$\pm$ 1.82}\tabularnewline
 & 20  & 40.24 {\small{}{}$\pm$ 1.97}  & \textbf{76.46}\textbf{\small{}{} $\pm$ 0.58}{\small{} } & 53.94 {\small{}{}$\pm$ 0.45}  & 72.72 {\small{}{}$\pm$ 0.88}  & 52.18 {\small{}{}$\pm$ 0.71}  & \uwave{73.65 }{\small{}\uwave{\mbox{$\pm$} 1.42}}\tabularnewline
 & 25  & 44.13 {\small{}{}$\pm$ 3.42}  & \textbf{75.95 }\textbf{\small{}{}$\pm$ 0.50}{\small{} } & 71.98 {\small{}{}$\pm$ 0.17}  & 74.91 {\small{}{}$\pm$ 0.56}  & 53.05 {\small{}{}$\pm$ 0.36}  & \uwave{75.36 }{\small{}\uwave{\mbox{$\pm$} 1.34}}\tabularnewline
\hline 
\multirow{6}{*}{ACM} & 0  & \uwave{91.75 }{\small{}\uwave{\mbox{$\pm$} 0.10}}  & 89.62 {\small{}{}$\pm$ 0.41}  & 87.51 {\small{}{}$\pm$ 0.42}  & 90.11 {\small{}{}$\pm$ 0.57}  & 91.45 {\small{}{}$\pm$ 0.21}  & \textbf{92.56 }\textbf{\small{}$\pm$ 0.42}\tabularnewline
 & 5  & 84.29 {\small{}{}$\pm$ 0.57}  & 84.64 {\small{}{}$\pm$ 0.27}  & 85.29 {\small{}{}$\pm$ 1.13}  & 88.25 {\small{}{}$\pm$ 1.19}  & \uwave{90.10 }{\small{}\uwave{\mbox{$\pm$} 0.27}}  & \textbf{90.60 }\textbf{\small{}{}$\pm$ 0.28}{\small{}{}}\tabularnewline
 & 10  & 81.71 {\small{}{}$\pm$ 0.61}  & 81.12 {\small{}{}$\pm$ 0.31}  & 84.59 {\small{}{}$\pm$ 0.68}  & 88.14 {\small{}{}$\pm$ 0.60}  & \uwave{89.45 }{\small{}\uwave{\mbox{$\pm$} 0.41}}  & \textbf{90.10 }\textbf{\small{}{}$\pm$ 0.35}{\small{}{}}\tabularnewline
 & 15  & 79.65 {\small{}{}$\pm$ 1.00}  & 74.66 {\small{}{}$\pm$ 0.94}  & 83.81 {\small{}{}$\pm$ 0.81}  & 87.59 {\small{}{}$\pm$ 0.74}  & \uwave{89.23 }{\small{}\uwave{\mbox{$\pm$} 0.34}}  & \textbf{89.93 }\textbf{\small{}{}$\pm$ 0.51}{\small{}{}}\tabularnewline
 & 20  & 79.95 {\small{}{}$\pm$ 0.50}  & 74.26 {\small{}{}$\pm$ 0.75}  & 82.35 {\small{}{}$\pm$ 1.64}  & 87.83 {\small{}{}$\pm$ 1.03}  & \uwave{88.65 }{\small{}\uwave{\mbox{$\pm$} 0.35}}  & \textbf{90.61 }\textbf{\small{}{}$\pm$ 0.28}{\small{}{}}\tabularnewline
 & 25  & 79.55 {\small{}{}$\pm$ 1.16}  & 74.12 {\small{}{}$\pm$ 0.81}  & 82.04 {\small{}{}$\pm$ 0.99}  & 88.06 {\small{}{}$\pm$ 0.85}  & \uwave{88.15 }{\small{}\uwave{\mbox{$\pm$} 0.58}}  & \textbf{90.15 }\textbf{\small{}{}$\pm$ 0.33}{\small{}{}}\tabularnewline
\hline 
\end{tabular}} 
\end{table}

\subsubsection{Global Attack}

We first test the node classification performance of ASGNN and other
baselines under global attack using a representative global attack
method called meta-attack \citep{zugner2019adversarial}. We vary
the perturbation rate, i.e., the ratio of changed edges, from 0\%
to 25\% with an increasing step of 5\%. The results are reported in
Table \ref{tab:global}. From the table, we observe that the proposed
ASGNN model outperforms other methods in most cases. For instance,
ASGNN improves GCN over 30\% on the Cora-ML dataset at a 20\% perturbation
rate and over 20\% on the Cora dataset at a 25\% perturbation rate.
On Cora, Citeseer, and ACM datasets, ASGNN beats other baselines at
various perturbation rates by a large margin. The GCN-Jaccard method
slightly outperforms ASGNN on the Cora-ML dataset at a 15\%-25\% perturbation
rate, while it performs poorly on other datasets. Specifically, on
the other three datasets under the 25\% perturbation rate, ASGNN outperforms
GCN-Jaccard by 22\%, 10\%, and 10\%, respectively. Such inspiring
results demonstrate that ASGNN can better resist global attack than
other baseline methods.

\subsubsection{Targeted Attack}

For the targeted attack, we use a representative method called NETTACK
\citep{zugner2018adversarial}. Following existing works \citep{zhu2019robust,jin2020graph},
we vary the perturbation number made on every node, i.e., the number
of edge removals/additions, from 0 to 5 with an increasing step of
1. The results are reported in Table \ref{tab:result-target}. We
choose the nodes in the test set with degrees larger than 10 as targeted
nodes and the reported classification performance is evaluated on
target nodes. Thus, the results in Table \ref{tab:result-target}
is not directly comparable with the results in Table \ref{tab:global}.
From the table, we can see that the proposed ASGNN attains better
performance than other baselines in most cases. For instance, on the
Citeseer dataset with 5 perturbations per targeted node, ASGNN improves
GCN by 25\% and outperforms other baselines by around 4\%. The reported
results demonstrate that ASGNN can also effectively resist the targeted
attack.

\begin{table}
\begin{centering}
\caption{Node classification performance (accuracy $\pm$ std) under targeted
attack (\textbf{Bold}: the best model; \uwave{wavy}: the runner-up
model)\label{tab:result-target}}
\par\end{centering}
\centering{}\resizebox{1 \textwidth}{!}{%
\begin{tabular}{c|c|cccccc}
\hline 
\textbf{Dataset}  & \textbf{Ptb. number}  & GCN  & GCN-Jaccard  & GCN-SVD  & Pro-GNN  & Elastic GNN  & ASGNN\tabularnewline
\hline 
\multirow{6}{*}{Cora} & 0  & 82.53 {\small{}$\pm$ 1.45}  & 81.95 {\small{}{}$\pm$ 0.29}  & 77.35 {\small{}{}$\pm$ 1.40}  & 82.92 {\small{}{}$\pm$ 0.29}  & \textbf{84.93 }\textbf{\small{}{}$\pm$ 2.28}{\small{} } & \uwave{83.01 }{\small{}\uwave{\mbox{$\pm$} 1.57}}\tabularnewline
 & 1  & 78.19 {\small{}$\pm$ 1.66}  & 75.30 {\small{}{}$\pm$ 1.54}  & 75.18 {\small{}{}$\pm$ 1.80}  & \uwave{81.48 }{\small{}\uwave{\mbox{$\pm$} 0.91}}  & 81.44 {\small{}{}$\pm$ 1.81}  & \textbf{81.57 }\textbf{\small{}$\pm$ 1.18}\tabularnewline
 & 2  & 71.33 {\small{}$\pm$ 1.29}  & 70.24 {\small{}{}$\pm$ 1.52}  & 71.81 {\small{}{}$\pm$ 1.63}  & \textbf{79.03 }\textbf{\small{}{}$\pm$ 1.80}{\small{} } & 76.74 {\small{}{}$\pm$ 1.97}  & \uwave{78.80 }{\small{}\uwave{\mbox{$\pm$} 1.03}}\tabularnewline
 & 3  & 66.63 {\small{}$\pm$ 1.53}  & 69.04 {\small{}{}$\pm$ 0.94}  & 65.18 {\small{}{}$\pm$ 1.65}  & 72.75 {\small{}{}$\pm$ 1.32}  & \uwave{73.97 }{\small{}\uwave{\mbox{$\pm$} 2.67}}  & \textbf{75.30 }\textbf{\small{}$\pm$ 1.35}\tabularnewline
 & 4  & 61.45 {\small{}$\pm$ 2.16}  & 61.68 {\small{}{}$\pm$ 1.05}  & 58.79 {\small{}{}$\pm$ 2.14}  & \uwave{70.11 }{\small{}\uwave{\mbox{$\pm$} 2.45}}  & 68.31 {\small{}$\pm$ 3.50}  & \textbf{70.24 }\textbf{\small{}$\pm$ 4.70}\tabularnewline
 & 5  & 56.75 {\small{}$\pm$ 1.37}  & 59.52 {\small{}{}$\pm$ 1.88}  & 59.16 {\small{}{}$\pm$ 2.71}  & \uwave{66.98 }{\small{}\uwave{\mbox{$\pm$} 1.63}}  & 65.78 {\small{}$\pm$ 2.51}  & \textbf{68.55 }\textbf{\small{}$\pm$ 3.21}\tabularnewline
\hline 
\multirow{6}{*}{Citeseer} & 0  & 81.27 {\small{}$\pm$ 0.95}  & 80.31 {\small{}{}$\pm$ 1.26}  & 80.47 {\small{}{}$\pm$ 1.01}  & 81.24 {\small{}$\pm$ 1.01}  & \uwave{81.42 }{\small{}\uwave{\mbox{$\pm$} 0.76}}  & \textbf{81.90 }\textbf{\small{}$\pm$ 1.95}\tabularnewline
 & 1  & 80.63 {\small{}$\pm$ 0.63}  & 80.00 {\small{}{}$\pm$ 1.45}  & 78.57 {\small{}{}$\pm$ 2.67}  & 80.52 {\small{}$\pm$ 0.85}  & \uwave{80.79 }{\small{}\uwave{\mbox{$\pm$} 1.17}}  & \textbf{81.21 }\textbf{\small{}$\pm$ 1.11}\tabularnewline
 & 2  & 79.84 {\small{}$\pm$ 1.02}  & 76.98 {\small{}{}$\pm$ 1.77}  & 73.02 {\small{}{}$\pm$ 6.77}  & 80.63 {\small{}$\pm$ 0.95}  & \uwave{81.01 }{\small{}\uwave{\mbox{$\pm$} 0.50}}  & \textbf{81.11 }\textbf{\small{}$\pm$ 1.32}\tabularnewline
 & 3  & 66.51 {\small{}{}$\pm$ 3.36}  & 74.76 {\small{}{}$\pm$ 1.31}  & 76.03 {\small{}{}$\pm$ 3.71}  & 79.36 {\small{}$\pm$ 4.76}  & \uwave{80.31 }{\small{}\uwave{\mbox{$\pm$} 1.10}}  & \textbf{80.32 }\textbf{\small{}$\pm$ 1.90}\tabularnewline
 & 4  & 62.54 {\small{}{}$\pm$ 1.62}  & 76.34 {\small{}{}$\pm$ 1.49}  & 62.22 {\small{}{}$\pm$ 3.31}  & \uwave{75.71 }{\small{}\uwave{\mbox{$\pm$} 4.87}}  & 72.06 {\small{}$\pm$ 5.60}  & \textbf{80.16 }\textbf{\small{}$\pm$ 1.28}\tabularnewline
 & 5  & 52.70 {\small{}{}$\pm$ 1.98}  & 72.85 {\small{}{}$\pm$ 1.65}  & 60.16 {\small{}{}$\pm$ 6.67}  & \uwave{73.95 }{\small{}\uwave{\mbox{$\pm$} 7.13}}  & 73.96 {\small{}$\pm$ 3.90}  & \textbf{77.94 }\textbf{\small{}$\pm$ 7.08}\tabularnewline
\hline 
\multirow{6}{*}{Cora-ML} & 0  & 88.33 {\small{}{}$\pm$ 0.56}  & 84.91 {\small{}{}$\pm$ 0.24}  & 83.06 {\small{}{}$\pm$ 0.30}  & 86.64 {\small{}$\pm$ 0.80}  & \uwave{88.84 }{\small{}\uwave{\mbox{$\pm$} 0.67}}  & \textbf{88.88 }\textbf{\small{}$\pm$ 0.31}\tabularnewline
 & 1  & 83.85 {\small{}{}$\pm$ 0.47}  & 83.82 {\small{}{}$\pm$ 0.33}  & 80.51 {\small{}{}$\pm$ 0.28}  & 83.77 {\small{}$\pm$ 0.96}  & \uwave{85.87 }{\small{}\uwave{\mbox{$\pm$} 1.01}}  & \textbf{86.20 }\textbf{\small{}$\pm$ 1.21}\tabularnewline
 & 2  & 79.18 {\small{}{}$\pm$ 1.31}  & 83.75 {\small{}{}$\pm$ 0.33}  & 78.73 {\small{}{}$\pm$ 0.34}  & 82.29 {\small{}$\pm$ 0.13}  & \uwave{84.34 }{\small{}\uwave{\mbox{$\pm$} 1.02}}  & \textbf{84.41 }\textbf{\small{}$\pm$ 1.51}\tabularnewline
 & 3  & 76.19 {\small{}{}$\pm$ 0.82}  & \uwave{82.18 }{\small{}\uwave{\mbox{$\pm$} 0.48}}  & 78.23 {\small{}{}$\pm$ 0.20}  & 80.58 {\small{}$\pm$ 1.06}  & 81.41 {\small{}$\pm$ 1.62}  & \textbf{83.50 }\textbf{\small{}$\pm$ 1.08}\tabularnewline
 & 4  & 70.61 {\small{}{}$\pm$ 1.56}  & \textbf{81.90 }\textbf{\small{}{}$\pm$ 0.45}{\small{} } & 76.25 {\small{}{}$\pm$ 0.51}  & 79.92 {\small{}$\pm$ 1.44}  & 78.31 {\small{}$\pm$ 1.20}  & \uwave{80.81 }{\small{}\uwave{\mbox{$\pm$} 0.99}}\tabularnewline
 & 5  & 64.52 {\small{}{}$\pm$ 1.29}  & \textbf{81.35 }\textbf{\small{}{}$\pm$ 0.34}{\small{} } & 75.47 {\small{}{}$\pm$ 0.41}  & 77.63 {\small{}$\pm$ 1.68}  & 74.08 {\small{}$\pm$ 1.89}  & \uwave{80.17 }{\small{}\uwave{\mbox{$\pm$} 3.85}}\tabularnewline
\hline 
\multirow{6}{*}{ACM} & 0  & \uwave{90.33 }{\small{}\uwave{\mbox{$\pm$} 0.09}}  & 89.76 {\small{}{}$\pm$ 0.39}  & 86.21 {\small{}{}$\pm$ 0.46}  & 90.22 {\small{}$\pm$ 0.60}  & 90.31 {\small{}$\pm$ 0.71}  & \textbf{92.11 }\textbf{\small{}$\pm$ 0.41}\tabularnewline
 & 1  & 89.70 {\small{}{}$\pm$ 0.22}  & 83.62 {\small{}{}$\pm$ 0.47}  & 83.50 {\small{}{}$\pm$ 0.71}  & 87.09 {\small{}$\pm$ 0.65}  & \uwave{90.03 }{\small{}\uwave{\mbox{$\pm$} 0.43}}  & \textbf{91.08 }\textbf{\small{}$\pm$ 1.36}\tabularnewline
 & 2  & 82.06 {\small{}{}$\pm$ 1.12}  & 80.47 {\small{}{}$\pm$ 0.63}  & 82.09 {\small{}{}$\pm$ 0.73}  & 87.01 {\small{}$\pm$ 0.53}  & \uwave{87.72 }{\small{}\uwave{\mbox{$\pm$} 1.27}}  & \textbf{90.55 }\textbf{\small{}$\pm$ 0.47}\tabularnewline
 & 3  & 80.26 {\small{}{}$\pm$ 1.17}  & 77.07 {\small{}{}$\pm$ 0.67}  & 81.09 {\small{}{}$\pm$ 0.78}  & \uwave{87.07 }{\small{}\uwave{\mbox{$\pm$} 1.14}}  & 84.75 {\small{}$\pm$ 2.07}  & \textbf{89.54 }\textbf{\small{}$\pm$ 0.49}\tabularnewline
 & 4  & 76.86 {\small{}{}$\pm$ 1.46}  & 77.45 {\small{}{}$\pm$ 0.44}  & 80.76 {\small{}{}$\pm$ 0.54}  & \uwave{87.04 }{\small{}\uwave{\mbox{$\pm$} 1.16}}  & 83.49 {\small{}$\pm$ 2.01}  & \textbf{89.45 }\textbf{\small{}$\pm$ 0.51}\tabularnewline
 & 5  & 73.32 {\small{}{}$\pm$ 1.77}  & 74.38 {\small{}{}$\pm$ 0.59}  & 79.74 {\small{}{}$\pm$ 0.77}  & \uwave{86.42 }{\small{}\uwave{\mbox{$\pm$} 0.71}}  & 81.67 {\small{}$\pm$ 1.53}  & \textbf{88.44 }\textbf{\small{}$\pm$ 0.71}\tabularnewline
\hline 
\end{tabular}} 
\end{table}

\subsubsection{Robustness of ASGNN}

The message passing scheme in ASGNN is designed based on jointly node
feature learning and graph structure learning principle. To validate
that ASGNN can help purify (i.e., denoise) the structure, we train
ASGNN models on perturbed graphs with various perturbation numbers
under NETTACK and evaluate their testing loss with the clean graph.
From the results in Table \ref{tab:robust}, we observe that ASGNN
achieves lower losses than GCN in all cases. Besides, the loss increases
slower as the perturbation number grows in ASGNN than in GCN. These
results indicate that ASGNN can help denoise the noisy graph structure.
\begin{table}
\begin{centering}
\caption{Testing loss on the clean graph of GNN models trained on perturbed
graphs with various perturbation numbers under targeted attack.\label{tab:robust}}
\par\end{centering}
\centering{}\resizebox{0.65 \textwidth}{!}{%
\begin{tabular}{ccccccc}
\hline 
\textbf{Ptb. number} & 0 & 1 & 2 & 3 & 4 & 5\tabularnewline
\hline 
GCN & 1.0315 & 1.3923 & 1.6703 & 2.0372 & 2.6108 & 3.3270\tabularnewline
ASGNN & 0.5534 & 0.7577 & 0.8323 & 0.8534 & 0.9040 & 1.0273\tabularnewline
\hline 
\end{tabular}}
\end{table}

According to Theorem \ref{ASMP: Convergence Theorem}, the convergence
of ASMP is guaranteed with proper choices of step sizes. Since the
step sizes are set to be learnable parameters, we provide an additional
experiment to evaluate the convergence of ASMP with learned step sizes
empirically in Appendix \ref{sec:Convergence-Property-of}.

\section{Conclusion }

In this work, we have developed an interpretable robust message passing
scheme named ASMP following the jointly node feature learning and
graph structure learning principle. ASMP is provably convergent and
it has a clear interpretation as a standard message passing scheme
with adaptive structure. Integrating ASMP with neural network components,
we have obtained a family of robust graph neural networks with adaptive
structure. Extensive experiments on real-world datasets with various
adversarial attack settings corroborate the effectiveness and the
robustness of the proposed graph neural network architecture. \bibliographystyle{plainnat}
\bibliography{ref}
 \newpage{}

\appendix

\section{Related Work on Optimization-Induced Graph Neural Network Design\label{Appendix:Related-Work}}

Since ASGNN proposed in this paper is induced from an optimization
algorithm, in this section, we give more literatue review on optimization-induced
GNN model design to supplement our discussion.

The idea of optimization-induced GNN model design partly stems from
the observation that many primitive handcrafted GNN models could be
nicely interpreted as (unrolled) iterative algorithms for solving
a GSD optimization problem \citep{ma2021unified,zhu2021interpreting,zhang2022towards}.
Based on this observation, many papers aim at strengthening the capability
of GNNs by carefully designing the underlying optimization problems
and/or the iterative algorithms solving it. 

For example, inspired by the idea of trend filtering \citep{wang2015trend},
\citet{liu2021elastic} replace the Laplacian smoothing term (which
is in the form of $\ell_{2}$-norm) in the GSD problem with an $\ell_{2,1}$-norm
to promote robustness against abnormal edges. Also for robustness
pursuit, \citet{yang2021graph} replace the Laplacian smoothing term
with nonlinear functions imposed over pairwise node distances. Since
the classical Laplacian smoothing term in GSD only promotes smoothness
over connected nodes, the authors in \citet{zhang2020revisiting,Model:PairNorm}
further suggest promoting the non-smoothness over the disconnected
nodes, which is achieved by deducting the sum of distances between
disconnected pairs of nodes from the denoising objective. In \citet{jiang2022fmp},
the authors augment the GSD objective with a fairness term to fight
against large topology bias. Most recently, \citet{fu2022p} propose
a $p$-Laplacian message passing scheme and a $^{p}$GNN model, which
is capable of dealing with heterophilic graphs and is robust to adversarial
perturbations. Apart from that, \citet{ahn2022descent} designs a
novel regularization term to build heterogeneous GNNs. 

Although there is rich literature on optimization-induced GNN model
design, all of them are focusing on learning the node feature matrix.
The idea of this paper is similar to them in terms of the GNN design
philosophy, however, we design an objective to jointly learn the node
feature and the graph structure which was rarely covered in the literature.

\section{Derivation of The Proximal Gradient Step in Eq. \eqref{ASMP: ProxGD Step}
\label{Appendix: Derivation of ProxGD in ASMP} }

For the $\mathbf{S}$-block optimization, i.e., Eq. \eqref{eq:S-objective},
we define the objective function except the term $\mu_{1}\left\Vert \mathbf{S}\right\Vert _{1}$
as $f_{S}(\mathbf{S})$, i.e.,
\begin{equation}
f_{S}(\mathbf{S})=\gamma\|\mathbf{S}-\mathbf{A}\|_{{\rm F}}^{2}-\lambda\mathrm{Tr}\left(\mathbf{H}^{\top}\mathbf{D}^{-1}\mathbf{S}\mathbf{H}\right)+\mu_{2}\left\Vert \mathbf{S}\right\Vert _{{\rm F}}^{2},\label{Appendix: Smooth Part of F_S}
\end{equation}
where $\mathbf{D}=\mathrm{Diag}\left(\mathbf{S}\mathbf{1}\right)$.
In this section, we first derive the expression of $\nabla f_{S}(\mathbf{S})$
and then compute the proximal operator in Eq. \eqref{ASMP: ProxGD Step}.

\subsection{On Computation of $\nabla f_{S}(\mathbf{S})$}

We first focus on the gradient computation of the second term in Eq.
\eqref{Appendix: Smooth Part of F_S}. For the graph degree matrix,
we have
\[
\mathbf{D}=\mathrm{Diag}\left(\mathbf{S}\mathbf{1}\right)=\mathbf{S}\mathbf{1}\mathbf{1}^{\top}\odot\mathbf{I},
\]
where $\odot$ denotes the Hadamard product. Based on the rule of
matrix calculus, the differential of the scalar function $\mathrm{Tr}(\mathbf{H}^{\top}\mathbf{D}^{-1}\mathbf{S}\mathbf{H})$
with respect to matrix variable $\mathbf{S}$ can be computed as follows:
\[
\mathrm{d}\left(\mathrm{Tr}\left(\mathbf{H}^{\top}\mathbf{D}^{-1}\mathbf{S}\mathbf{H}\right)\right)=\mathrm{Tr}\left(\mathbf{H}^{\top}\mathbf{D}^{-1}\mathrm{d}\left(\mathbf{S}\right)\mathbf{H}+\mathbf{H}^{\top}\mathrm{d}\left(\mathbf{D}^{-1}\right)\mathbf{S}\mathbf{H}\right).
\]
For an invertible $\mathbf{D}$ (note that, in this paper, the graphs
considered all have self loops, so $\mathbf{D}$ is always invertible),
we have
\[
\mathrm{d}\bigl(\mathbf{D}^{-1}\bigr)=-\mathbf{D}^{-1}\mathrm{d}\left(\mathbf{D}\right)\mathbf{D}^{-1}.
\]
Thus, we can get
\begin{align*}
\mathrm{d}\left(\mathrm{Tr}\left(\mathbf{H}^{\top}\mathbf{D}^{-1}\mathbf{S}\mathbf{H}\right)\right)= & \mathrm{Tr}\left(\mathbf{H}^{\top}\mathbf{D}^{-1}\mathrm{d}\left(\mathbf{S}\right)\mathbf{H}-\mathbf{H}^{\top}\mathbf{D}^{-1}\mathrm{d}\bigl(\mathbf{S}\mathbf{1}\mathbf{1}^{\top}\odot\mathbf{I}\bigr)\mathbf{D}^{-1}\mathbf{S}\mathbf{H}\right)\\
= & \mathrm{Tr}\left(\mathbf{H}\mathbf{H}^{\top}\mathbf{D}^{-1}\mathrm{d}\left(\mathbf{S}\right)-\left(\mathbf{D}^{-1}\mathbf{S}\mathbf{H}\mathbf{H}^{\top}\mathbf{D}^{-1}\odot\mathbf{I}\right)\mathrm{d}\left(\mathbf{S}\right)\mathbf{1}\mathbf{1}^{\top}\right)\\
= & \mathrm{Tr}\left(\Bigl(\mathbf{H}\mathbf{H}^{\top}\mathbf{D}^{-1}-\mathbf{1}\mathrm{Diag}\left(\mathbf{D}^{-1}\mathbf{S}\mathbf{H}\mathbf{H}^{\top}\mathbf{D}^{-1}\right)^{\top}\Bigr)\mathrm{d}\left(\mathbf{S}\right)\right).
\end{align*}
Since $\mathrm{d}\left(\mathrm{Tr}(\mathbf{H}^{\top}\mathbf{D}^{-1}\mathbf{S}\mathbf{H})\right)=\mathrm{Tr}\Bigl(\bigl(\frac{\mathrm{d}\mathrm{Tr}(\mathbf{H}^{\top}\mathbf{D}^{-1}\mathbf{S}\mathbf{H})}{\mathrm{d}\mathbf{S}}\bigr)^{\top}\mathrm{d}(\mathbf{S})\Bigr)$,
we have
\begin{align*}
\frac{\mathrm{d}\mathrm{Tr}\left(\mathbf{H}^{\top}\mathbf{D}^{-1}\mathbf{S}\mathbf{H}\right)}{\mathrm{d}\mathbf{S}} & =\mathbf{D}^{-1}\mathbf{H}\mathbf{H}^{\top}-\mathrm{Diag}\left(\mathbf{D}^{-1}\mathbf{S}\mathbf{H}\mathbf{H}^{\top}\mathbf{D}^{-1}\right)\mathbf{1}^{\top}.
\end{align*}
For other terms in $f_{S}(\mathbf{S})$, the gradients with respect
to $\mathbf{S}$ can be easily computed. Finally, we obtain
\[
\nabla f_{S}(\mathbf{S})=2\gamma\left(\mathbf{S}-\mathbf{A}\right)-\lambda\left(\mathbf{D}^{-1}\mathbf{H}\mathbf{H}^{\top}-\mathrm{Diag}\left(\mathbf{D}^{-1}\mathbf{S}\mathbf{H}\mathbf{H}^{\top}\mathbf{D}^{-1}\right)\mathbf{1}^{\top}\right)+2\mu_{2}\mathbf{S}.
\]

\subsection{On Computation of The Proximal Step}
\begin{lem}
\label{lem:The-proximal}Given a matrix $\mathbf{M}\in\mathbb{R}^{N\times N}$,
we have
\begin{equation}
\mathrm{prox}_{\kappa\left\Vert \cdot\right\Vert _{1}+\mathbb{I}_{\mathcal{S}}\left(\cdot\right)}\left(\mathbf{M}\right)=\min\left\{ 1,\mathrm{ReLU}\bigl(\mathbf{M}-\kappa\mathbf{1}\mathbf{1}^{\top}\bigr)\right\} ,\label{eq:prox-solution}
\end{equation}
where $\mathrm{ReLU}(\mathbf{X})=\max\{\mathbf{0},\mathbf{X}\}$.
\end{lem}

\begin{proof}
The proximal step in Eq. \eqref{eq:prox-solution} can be rewritten
as the following optimization problem:
\begin{equation}
\underset{\mathbf{S}\in\mathcal{S}}{\mathsf{minimize}}\ \ \ \frac{1}{2}\left\Vert \mathbf{S}-\mathbf{M}\right\Vert _{{\rm F}}^{2}+\kappa\left\Vert \mathbf{S}\right\Vert _{1}.\label{eq:prox-prob}
\end{equation}
It is easy to observe that Problem \eqref{eq:prox-prob} is decoupled
over different elements in matrix $\mathbf{S}$. Therefore, each $\mathbf{S}_{ij}$
with $i,j=1,\ldots,N$ can be optimized individually by solving the
following optimization problem:
\begin{equation}
\underset{0\leq\mathbf{S}_{ij}\leq1}{\mathsf{minimize}}\ \ \ h_{ij}(\mathbf{S}_{ij})=\frac{1}{2}\left\Vert \mathbf{S}_{ij}-\mathbf{M}_{ij}\right\Vert _{{\rm F}}^{2}+\kappa\left|\mathbf{S}_{ij}\right|.\label{eq:prox-prob-element}
\end{equation}
According to Eq. \eqref{eq:prox-solution}, we have
\[
\begin{aligned}\mathbf{S}_{ij}^{\star} & =\mathsf{min}\left\{ 1,\mathrm{ReLU}\left(\mathbf{M}_{ij}-\kappa\right)\right\} \\
 & =\begin{cases}
1 & 1+\kappa\leq\mathbf{M}_{ij}\\
\mathbf{M}_{ij}-\kappa & \kappa\leq\mathbf{M}_{ij}<1+\kappa\\
0 & \mathbf{M}_{ij}<\kappa.
\end{cases}
\end{aligned}
\]
Then, Lemma \ref{lem:The-proximal} can be proved by showing that
for $i,$ $j=1,\ldots N$, $\mathbf{S}_{ij}^{\star}$ is the optimal
solution for Problem \eqref{eq:prox-prob-element}. The optimality
of $\mathbf{S}_{ij}^{\star}$ can be validated by verifying the optimality
condition, i.e., there exists a subgradient $\psi\in\partial h_{ij}\bigl(\mathbf{S}_{ij}^{\star}\bigr)$
such that
\[
\psi\bigl(\mathbf{S}_{ij}-\mathbf{S}_{ij}^{\star}\bigr)\geq0
\]
for all $0\leq\mathbf{S}_{ij}\leq1$. Observe that the subdifferential
of $h_{ij}(\mathbf{S}_{ij})$ is computed as follows: 
\[
\partial h_{ij}(\mathbf{S}_{ij})=\begin{cases}
\mathbf{S}_{ij}-\mathbf{M}_{ij}+\kappa & \mathbf{S}_{ij}>0\\
\mathbf{S}_{ij}-\mathbf{M}_{ij}+\kappa\epsilon & \mathbf{S}_{ij}=0,
\end{cases}
\]
where $\epsilon$ can be any constant satisfying $-1\leq\epsilon\leq1$.
Then, the subdifferential $\partial h\bigl(\mathbf{S}_{ij}^{\star}\bigr)$
is given by
\[
\partial h_{ij}(\mathbf{S}_{ij}^{\star})=\begin{cases}
1-\mathbf{M}_{ij}+\kappa & 1+\kappa\leq\mathbf{M}_{ij}\\
0 & \kappa\leq\mathbf{M}_{ij}<1+\kappa\\
-\mathbf{M}_{ij}+\kappa\epsilon & \mathbf{M}_{ij}<\kappa.
\end{cases}
\]
In the following, we will show that the optimality condition holds
for each of the above cases.
\begin{enumerate}
\item For $1+\kappa\leq\mathbf{M}_{ij}$, we have $\mathbf{S}_{ij}^{\star}=1$
and $\psi=1-\mathbf{M}_{ij}+\kappa\leq0$. Since $\mathbf{S}_{ij}\leq1$,
we can get $\psi(\mathbf{S}_{ij}-\mathbf{S}_{ij}^{\star})\geq0$.
\item For $\kappa\leq\mathbf{M}_{ij}<1+\kappa$, we have $\psi=0$ and hence,
$\psi(\mathbf{S}_{ij}-\mathbf{S}_{ij}^{\star})=0$ for all $0\leq\mathbf{S}_{ij}\leq1$.
\item For $\mathbf{M}_{ij}<\kappa$, we have $\mathbf{S}_{ij}^{\star}=0$
and $\psi=-\mathbf{M}_{ij}+\kappa\epsilon$ with $\epsilon$ being
any constant satisfying $-1\leq\epsilon\leq1$. Thus, we can choose
$\epsilon=1$, which leads to $\psi>0$. Since $\mathbf{S}_{ij}\geq0$,
we can get $\psi(\mathbf{S}_{ij}-\mathbf{S}_{ij}^{\star})\geq0$.
\end{enumerate}
In conclusion, there exists a subgradient $\psi\in\partial h_{ij}(\mathbf{S}_{ij}^{\star})$
such that $\psi(\mathbf{S}_{ij}-\mathbf{S}_{ij}^{\star})\geq0$ for
all $0\leq\mathbf{S}_{ij}\leq1$, based on which the optimality of
Eq. \eqref{eq:prox-solution} is validated and the proof is completed.
\end{proof}
Based on the result in Lemma \ref{lem:The-proximal}, by choosing
$\mathbf{M}=\mathbf{S}^{(k)}-\eta_{2}\nabla f_{S}(\mathbf{S}^{(k)})$
and $\kappa=\eta_{2}\mu_{1}$, we get the analytical expression for
the proximal step in Eq. \eqref{ASMP: ProxGD Step} as follows:
\[
\mathbf{S}^{(k+1)}=\min\left\{ 1,\mathrm{ReLU}\Bigl(\mathbf{S}^{(k)}-\eta_{2}\nabla f_{S}(\mathbf{S}^{(k)})-\eta_{2}\mu_{1}\mathbf{1}\mathbf{1}^{\top}\Bigr)\right\} ,
\]
which suffices to perform the soft-thresholding operation and then
project the solution onto the constraint set $\mathcal{S}$.

\section{Proof of Theorem \ref{ASMP: Convergence Theorem} (Convergence of
ASMP) \label{Appendix: Convergence Proof}}

In this section, we will first prove that the objective function at
the $\mathbf{H}$-block optimization problem and the smooth part of
the objective function at the $\mathbf{S}$-block optimization problem
are $L$-smooth. Then we give the conditions to ensure convergence
of ASMP. 

Denote $f_{H}\left(\mathbf{H}\right)$ as the objective function at
the $\mathbf{H}$-block optimization problem, i.e., 
\[
f_{H}\left(\mathbf{H}\right)=\|\mathbf{H}-\mathbf{X}\|_{{\rm F}}^{2}+\lambda\mathrm{Tr}\left(\mathbf{H}^{\top}\mathbf{L}_{{\rm rw}}\mathbf{H}\right).
\]
The $L$-smoothness of $f_{H}\left(\mathbf{H}\right)$ is demonstrated
in the following lemma.
\begin{lem}
\label{lem:L-smooth-H}Function $f_{H}(\mathbf{H})$ is $L$-smooth
with $L_{H}=2+4\lambda$, i.e., for any $\mathbf{H}_{1},\mathbf{H}_{2}\in\mathbb{R}^{N\times M}$,
the following inequality holds:
\[
\left\Vert \nabla f_{H}(\mathbf{H}_{1})-\nabla f_{H}(\mathbf{H}_{2})\right\Vert _{{\rm F}}\leq L_{H}\left\Vert \mathbf{H}_{1}-\mathbf{H}_{2}\right\Vert _{{\rm F}}.
\]
\end{lem}

\begin{proof}
First observe that
\begin{align*}
 & \left\Vert \nabla f_{H}(\mathbf{H}_{1})-\nabla f_{H}(\mathbf{H}_{2})\right\Vert _{{\rm F}}\\
= & \left\Vert 2\left(\mathbf{H}_{1}-\mathbf{X}\right)+2\lambda\mathbf{L}_{{\rm rw}}\mathbf{H}_{1}-\bigl(2\left(\mathbf{H}_{2}-\mathbf{X}\right)+2\lambda\mathbf{L}_{{\rm rw}}\mathbf{H}_{2}\bigr)\right\Vert _{{\rm F}}\\
= & 2\left\Vert \left(\mathbf{I}+\lambda\mathbf{L}_{{\rm rw}}\right)\left(\mathbf{H}_{1}-\mathbf{H}_{2}\right)\right\Vert _{{\rm F}}\\
\leq & 2\left\Vert \mathbf{I}+\lambda\mathbf{L}_{{\rm rw}}\right\Vert _{2}\left\Vert \mathbf{H}_{1}-\mathbf{H}_{2}\right\Vert _{{\rm F}}.
\end{align*}
\begin{lem}[\citealp{chung1997spectral}]
\label{lem:The-largest-eigenvalue}The largest eigenvalue of a random
walk normalized Laplacian matrix $\mathbf{L}_{{\rm rw}}$ is less
than or equal to 2, i.e., $||\mathbf{L}_{\mathrm{rw}}||_{2}\leq2$.
\end{lem}

Based on Lemma \ref{lem:The-largest-eigenvalue}, we can conclude
that 
\[
\left\Vert \nabla f_{H}(\mathbf{H}_{1})-\nabla f_{H}(\mathbf{H}_{2})\right\Vert _{{\rm F}}\leq\left(2+4\lambda\right)\left\Vert \mathbf{H}_{1}-\mathbf{H}_{2}\right\Vert _{{\rm F}}.
\]
Therefore, function $f_{H}(\mathbf{H})$ is $L$-smooth with $L_{H}=2+4\lambda$
and the proof is completed
\end{proof}
With $f_{S}$ defined in Eq. \eqref{Appendix: Smooth Part of F_S},
the $L$-smoothness of $f_{S}$ is deomnstrated in the following lemma.
\begin{lem}
\label{lem:L-smooth-S}Function $f_{S}(\mathbf{S})$ is $L$-smooth
with $L_{S}=2\gamma+2\mu_{2}+\frac{2\lambda}{c^{2}}N^{2}B^{2}+\frac{2\lambda}{c^{3}}N^{3}\sqrt{N}B^{2}$,
i.e., for any $\mathbf{S}_{1},\mathbf{S}_{2}\in\mathbb{R}^{N\times N}$,
the following inequality holds:
\[
\left\Vert \nabla f_{S}(\mathbf{S}_{1})-\nabla f_{S}(\mathbf{S}_{2})\right\Vert _{{\rm F}}\leq L_{S}\left\Vert \mathbf{S}_{1}-\mathbf{S}_{2}\right\Vert _{{\rm F}}.
\]
\end{lem}

\begin{proof}
Denote $\mathbf{D}_{1}=\mathrm{Diag}\left(\mathbf{S}_{1}\mathbf{1}\right)$
and $\mathbf{D}_{2}=\mathrm{Diag}\left(\mathbf{S}_{2}\mathbf{1}\right)$
as two degree matrices corresponding to $\mathbf{S}_{1}$ and $\mathbf{S}_{2}$.
We have
\begin{equation}
\begin{aligned} & \left\Vert \nabla f_{S}(\mathbf{S}_{1})-\nabla f_{S}(\mathbf{S}_{2})\right\Vert _{{\rm F}}\\
= & \Bigl\Vert2\gamma\left(\mathbf{S}_{1}-\mathbf{A}\right)+2\mu_{2}\mathbf{S}_{1}-\lambda\mathbf{D}_{1}^{-1}\mathbf{H}\mathbf{H}^{\top}+\lambda\mathrm{Diag}\left(\mathbf{D}_{1}^{-1}\mathbf{S}_{1}\mathbf{H}\mathbf{H}^{\top}\mathbf{D}_{1}^{-1}\right)\mathbf{1}^{\top}\\
 & \hspace{1.15cm}-\left(2\gamma\left(\mathbf{S}_{2}-\mathbf{A}\right)+2\mu_{2}\mathbf{S}_{2}-\lambda\mathbf{D}_{2}^{-1}\mathbf{H}\mathbf{H}^{\top}+\lambda\mathrm{Diag}\left(\mathbf{D}_{2}^{-1}\mathbf{S}_{2}\mathbf{H}\mathbf{H}^{\top}\mathbf{D}_{2}^{-1}\right)\mathbf{1}^{\top}\right)\Bigr\Vert_{{\rm F}}\\
\leq & \left(2\gamma+2\mu_{2}\right)\left\Vert \mathbf{S}_{1}-\mathbf{S}_{2}\right\Vert _{{\rm F}}+\lambda\left\Vert \left(\mathbf{D}_{1}^{-1}-\mathbf{D}_{2}^{-1}\right)\mathbf{H}\mathbf{H}^{\top}\right\Vert _{{\rm F}}\\
 & \hspace{3.5cm}+\lambda\left\Vert \mathrm{Diag}\left(\mathbf{D}_{1}^{-1}\mathbf{S}_{1}\mathbf{H}\mathbf{H}^{\top}\mathbf{D}_{1}^{-1}-\mathbf{D}_{2}^{-1}\mathbf{S}_{2}\mathbf{H}\mathbf{H}^{\top}\mathbf{D}_{2}^{-1}\right)\mathbf{1}^{\top}\right\Vert _{{\rm F}}\\
\leq & \left(2\gamma+2\mu_{2}\right)\left\Vert \mathbf{S}_{1}-\mathbf{S}_{2}\right\Vert _{{\rm F}}+\lambda\left\Vert \mathbf{H}\mathbf{H}^{\top}\right\Vert _{2}\left\Vert \mathbf{D}_{1}^{-1}-\mathbf{D}_{2}^{-1}\right\Vert _{{\rm F}}\\
 & \hspace{3.5cm}+\lambda\sqrt{N}\left\Vert \mathbf{D}_{1}^{-1}\mathbf{S}_{1}\mathbf{H}\mathbf{H}^{\top}\mathbf{D}_{1}^{-1}-\mathbf{D}_{2}^{-1}\mathbf{S}_{2}\mathbf{H}\mathbf{H}^{\top}\mathbf{D}_{2}^{-1}\right\Vert _{{\rm F}}.
\end{aligned}
\label{eq:gradient f(S)-f(S)}
\end{equation}
To derive the Lipschitz constant of $\nabla f_{S}(\mathbf{S})$, we
first present several useful lemmas.
\begin{lem}
\label{lem:HH^T bound}Under Assumption \ref{assu:feature norm} that
the norm of node feature vectors is upperbounded, i.e., $\bigl\Vert\mathbf{H}_{i,:}\bigr\Vert_{2}\leq B$,
we have
\[
\left\Vert \mathbf{H}\mathbf{H}^{\top}\right\Vert _{2}\leq\left\Vert \mathbf{H}\mathbf{H}^{\top}\right\Vert _{{\rm F}}=\sqrt{\sum_{i=1}^{N}\sum_{j=1}^{N}\left(\mathbf{H}_{i,:}^{\top}\mathbf{H}_{j}\right)^{2}}\leq\sqrt{\sum_{i=1}^{N}\sum_{j=1}^{N}B^{4}}=NB^{2}.
\]
\end{lem}

\begin{lem}
\label{lem:DS bound}Given $\mathbf{S}\in\mathcal{S}$ and $\mathbf{D}=\mathrm{Diag}\left(\mathbf{S}\mathbf{1}\right)$,
under Assumption \ref{assu:degree matrix} that the diagonal elements
of $\mathbf{D}$ is lowerbounded by a positive constant, i.e., $\min_{i}\mathbf{D}_{ii}=c>0$
for $i=1,\ldots,N$, we have
\[
\left\Vert \mathbf{D}^{-1}\mathbf{S}\right\Vert _{2}\leq\left\Vert \mathbf{D}^{-1}\mathbf{S}\right\Vert _{{\rm F}}=\sqrt{\sum_{i=1}^{N}\sum_{j=1}^{N}\left(\frac{\mathbf{S}_{ij}}{\mathbf{D}_{ii}}\right)^{2}}\leq\frac{N}{c}.
\]
\end{lem}

\begin{lem}
\label{lem:D1-D2 bound}Given $\mathbf{S}_{1}$, $\mathbf{S}_{2}\in\mathcal{S}$,
$\mathbf{D}_{1}=\mathrm{Diag}\left(\mathbf{S}_{1}\mathbf{1}\right)$,
and $\mathbf{D}_{2}=\mathrm{Diag}\left(\mathbf{S}_{2}\mathbf{1}\right)$,
under Assumption \ref{assu:degree matrix} that the diagonal elements
of the degree matrix is lowerbounded by a positive constant, i.e.,
$\min_{i}\mathbf{D}_{ii}=c>0$ for $i=1,\ldots,N$, we have
\begin{equation}
\left\Vert \mathbf{D}_{1}^{-1}-\mathbf{D}_{2}^{-1}\right\Vert _{{\rm F}}\leq\frac{1}{c^{2}}N\left\Vert \mathbf{S}_{1}-\mathbf{S}_{2}\right\Vert _{{\rm F}},\label{eq:D1-D2 bound}
\end{equation}
and
\begin{equation}
\left\Vert \mathbf{D}_{1}^{-1}\mathbf{S}_{1}-\mathbf{D}_{2}^{-1}\mathbf{S}_{2}\right\Vert _{2}\leq\Bigl(\frac{1}{c}+\frac{1}{c^{2}}N^{2}\Bigr)\left\Vert \mathbf{S}_{1}-\mathbf{S}_{2}\right\Vert _{{\rm F}}.\label{eq:D1S-D2S bound}
\end{equation}
\end{lem}

\begin{proof}
It can be observed that
\begin{align*}
\left\Vert \mathbf{D}_{1}^{-1}-\mathbf{D}_{2}^{-1}\right\Vert _{{\rm F}} & =\sqrt{\sum_{i=1}^{N}\left(\frac{\left[\mathbf{D}_{1}\right]_{ii}-\left[\mathbf{D}_{2}\right]_{ii}}{\left[\mathbf{D}_{1}\right]_{ii}\left[\mathbf{D}_{2}\right]_{ii}}\right)^{2}}\\
 & \leq\frac{1}{c^{2}}\sqrt{\sum_{i=1}^{N}\left(N\max_{j=1,\ldots,N}\left|\left[\mathbf{S}_{1}\right]_{ij}-\left[\mathbf{S}_{2}\right]_{ij}\right|\right)^{2}}\\
 & \leq\frac{1}{c^{2}}N\left\Vert \mathbf{S}_{1}-\mathbf{S}_{2}\right\Vert _{{\rm F}},
\end{align*}
which proves Eq. \eqref{eq:D1-D2 bound}. Based on Eq. \eqref{eq:D1-D2 bound},
we further have
\[
\begin{aligned}\left\Vert \mathbf{D}_{1}^{-1}\mathbf{S}_{1}-\mathbf{D}_{2}^{-1}\mathbf{S}_{2}\right\Vert _{2} & \leq\left\Vert \mathbf{D}_{1}^{-1}\mathbf{S}_{1}-\mathbf{D}_{2}^{-1}\mathbf{S}_{2}\right\Vert _{{\rm F}}\\
 & \leq\left\Vert \mathbf{D}_{1}^{-1}\mathbf{S}_{1}-\mathbf{D}_{1}^{-1}\mathbf{S}_{2}+\mathbf{D}_{1}^{-1}\mathbf{S}_{2}-\mathbf{D}_{2}^{-1}\mathbf{S}_{2}\right\Vert _{{\rm F}}\\
 & \leq\left\Vert \mathbf{D}_{1}^{-1}\right\Vert _{2}\left\Vert \mathbf{S}_{1}-\mathbf{S}_{2}\right\Vert _{{\rm F}}+\left\Vert \mathbf{S}_{2}\right\Vert _{2}\left\Vert \mathbf{D}_{1}^{-1}-\mathbf{D}_{2}^{-1}\right\Vert _{{\rm F}}\\
 & \leq\left\Vert \mathbf{D}_{1}^{-1}\right\Vert _{2}\left\Vert \mathbf{S}_{1}-\mathbf{S}_{2}\right\Vert _{{\rm F}}+N\left\Vert \mathbf{D}_{1}^{-1}-\mathbf{D}_{2}^{-1}\right\Vert _{{\rm F}}\\
 & \leq\Bigl(\frac{1}{c}+\frac{1}{c^{2}}N^{2}\Bigr)\left\Vert \mathbf{S}_{1}-\mathbf{S}_{2}\right\Vert _{{\rm F}},
\end{aligned}
\]
through which the proof is completed.
\end{proof}
Based on Lemma \ref{lem:HH^T bound} and Lemma \ref{lem:D1-D2 bound},
the second term in Eq. \eqref{eq:gradient f(S)-f(S)}, i.e., $\lambda\left\Vert \mathbf{H}\mathbf{H}^{\top}\right\Vert _{2}\left\Vert \mathbf{D}_{1}^{-1}-\mathbf{D}_{2}^{-1}\right\Vert _{{\rm F}}$,
can be upperbounded as follows:
\begin{equation}
\lambda\left\Vert \mathbf{H}\mathbf{H}^{\top}\right\Vert _{2}\left\Vert \mathbf{D}_{1}^{-1}-\mathbf{D}_{2}^{-1}\right\Vert _{{\rm F}}\leq\frac{\lambda}{c^{2}}N^{2}B^{2}\left\Vert \mathbf{S}_{1}-\mathbf{S}_{2}\right\Vert _{{\rm F}}.\label{eq:second term bound}
\end{equation}
For the third term in Eq. \eqref{eq:gradient f(S)-f(S)}, we have
\[
\begin{aligned} & \lambda\sqrt{N}\left\Vert \mathbf{D}_{1}^{-1}\mathbf{S}_{1}\mathbf{H}\mathbf{H}^{\top}\mathbf{D}_{1}^{-1}-\mathbf{D}_{2}^{-1}\mathbf{S}_{2}\mathbf{H}\mathbf{H}^{\top}\mathbf{D}_{2}^{-1}\right\Vert _{{\rm F}}\\
= & \lambda\sqrt{N}\left\Vert \mathbf{D}_{1}^{-1}\mathbf{S}_{1}\mathbf{H}\mathbf{H}^{\top}\mathbf{D}_{1}^{-1}-\mathbf{D}_{1}^{-1}\mathbf{S}_{1}\mathbf{H}\mathbf{H}^{\top}\mathbf{D}_{2}^{-1}+\mathbf{D}_{1}^{-1}\mathbf{S}_{1}\mathbf{H}\mathbf{H}^{\top}\mathbf{D}_{2}^{-1}-\mathbf{D}_{2}^{-1}\mathbf{S}_{2}\mathbf{H}\mathbf{H}^{\top}\mathbf{D}_{2}^{-1}\right\Vert _{{\rm F}}\\
\leq & \lambda\sqrt{N}\left\Vert \mathbf{D}_{1}^{-1}\mathbf{S}_{1}\mathbf{H}\mathbf{H}^{\top}\bigl(\mathbf{D}_{1}^{-1}-\mathbf{D}_{2}^{-1}\bigr)\right\Vert _{{\rm F}}+\lambda\sqrt{N}\left\Vert \bigl(\mathbf{D}_{1}^{-1}\mathbf{S}_{1}-\mathbf{D}_{2}^{-1}\mathbf{S}_{2}\bigr)\mathbf{H}\mathbf{H}^{\top}\mathbf{D}_{2}^{-1}\right\Vert _{{\rm F}}\\
\leq & \lambda\sqrt{N}\left\Vert \mathbf{D}_{1}^{-1}\mathbf{S}_{1}\right\Vert _{2}\left\Vert \mathbf{H}\mathbf{H}^{\top}\right\Vert _{{\rm F}}\left\Vert \mathbf{D}_{1}^{-1}-\mathbf{D}_{2}^{-1}\right\Vert _{2}+\lambda\sqrt{N}\left\Vert \mathbf{D}_{1}^{-1}\mathbf{S}_{1}-\mathbf{D}_{2}^{-1}\mathbf{S}_{2}\right\Vert _{2}\left\Vert \mathbf{H}\mathbf{H}^{\top}\right\Vert _{{\rm F}}\left\Vert \mathbf{D}_{2}^{-1}\right\Vert _{2}.
\end{aligned}
\]
Based on Lemma \ref{lem:HH^T bound}, Lemma \ref{lem:DS bound}, and
Lemma \ref{lem:D1-D2 bound}, we can get the following result:
\begin{equation}
\lambda\sqrt{N}\left\Vert \mathbf{D}_{1}^{-1}\mathbf{S}_{1}\mathbf{H}\mathbf{H}^{\top}\mathbf{D}_{1}^{-1}-\mathbf{D}_{2}^{-1}\mathbf{S}_{2}\mathbf{H}\mathbf{H}^{\top}\mathbf{D}_{2}^{-1}\right\Vert _{{\rm F}}\leq\Bigl(1+\frac{2}{c}N^{2}\Bigr)\frac{\lambda}{c^{2}}N\sqrt{N}B^{2}\left\Vert \mathbf{S}_{1}-\mathbf{S}_{2}\right\Vert _{{\rm F}}.\label{eq:third term bound}
\end{equation}
Substituting the results in Eq. \eqref{eq:second term bound} and
Eq. \eqref{eq:third term bound} into Eq. \eqref{eq:gradient f(S)-f(S)}
gives
\begin{align*}
\left\Vert \nabla f_{S}(\mathbf{S}_{1})-\nabla f_{S}(\mathbf{S}_{2})\right\Vert _{{\rm F}} & \leq\left(2\gamma+2\mu_{2}+\frac{\lambda}{c^{2}}N^{2}B^{2}+\bigl(1+\frac{2}{c}N^{2}\bigr)\frac{\lambda}{c^{2}}N\sqrt{N}B^{2}\right)\left\Vert \mathbf{S}_{1}-\mathbf{S}_{2}\right\Vert _{{\rm F}}\\
 & \leq\left(2\gamma+2\mu_{2}+\bigl(1+\frac{1}{c}N\sqrt{N}\bigr)\frac{2\lambda}{c^{2}}N^{2}B^{2}\right)\left\Vert \mathbf{S}_{1}-\mathbf{S}_{2}\right\Vert _{{\rm F}}.
\end{align*}
Therefore, function $f_{S}(\mathbf{S})$ is $L$-smooth with $L_{S}=2\gamma+2\mu_{2}+\frac{2\lambda}{c^{2}}N^{2}B^{2}+\frac{2\lambda}{c^{3}}N^{3}\sqrt{N}B^{2}$
and the proof is completed.
\end{proof}
Based on the results in Lemma \ref{lem:L-smooth-H} and Lemma \ref{lem:L-smooth-S},
we can conclude that $f_{H}(\mathbf{H})$ and $f_{S}(\mathbf{S})$
are both $L$-smooth. To ensure the monotonically decreasing property
of \eqref{ASMP: Procedure}, the step sizes must satisfy \citep{parikh2014proximal}:
\[
0<\eta_{1}<\frac{2}{L_{H}}=\frac{1}{1+2\lambda}\quad\text{and}\quad0<\eta_{2}<\frac{2}{L_{S}}=\frac{1}{\gamma+\mu_{2}+\bigl(1+\frac{1}{c}N\sqrt{N}\bigr)\frac{\lambda}{c^{2}}N^{2}B^{2}}.
\]
Under such condition, the convergence of \eqref{ASMP: Procedure}
to a first-order stationary point of Problem \eqref{eq:Structure Signal Denoising}
can be readily obtained based on the results for alternating proximal
gradient descent method in \citet{bolte2014proximal,nikolova2017alternating}
with convergence rate
\begin{equation}
\inf_{k\geq K}\left\{ \bigl\Vert\mathbf{H}^{(k+1)}-\mathbf{H}^{(k)}\bigr\Vert_{{\rm F}}^{2}+\bigl\Vert\mathbf{S}^{(k+1)}-\mathbf{S}^{(k)}\bigr\Vert_{{\rm F}}^{2}\right\} \leq\frac{1}{\rho K}\left(p\bigl(\mathbf{H}^{(0)},\mathbf{S}^{(0)}\bigr)-p\bigl(\mathbf{H}^{\ast},\mathbf{S}^{\ast}\bigr)\right),\label{eq:alternating-rate}
\end{equation}
where $\rho=\min\bigl\{\frac{1}{\eta_{1}}-\frac{L_{H}}{2},\frac{1}{\eta_{2}}-\frac{L_{S}}{2}\bigr\}$
and $p\left(\mathbf{H},\mathbf{S}\right)$ represents the objective
function in Eq. \eqref{eq:Structure Signal Denoising}.

\section{Discussion on the Joint Optimization Approach\label{sec:Discussion-on-joint}}

In this paper, we have used the alternating optimization approach
to induce the ASMP scheme, while another natural idea is to apply
a joint optimization approach for Problem \eqref{eq:Structure Signal Denoising}.
In this section, we will show that the joint optimization approach
actually is inferior compared to the alternating one, since joint
optimization would lead to slower convergence, which motivates the
use of alternating optimization in ASMP.

We define the smooth part of the objective in Problem \eqref{eq:Structure Signal Denoising}
as
\[
f\left(\mathbf{H},\mathbf{S}\right)=\|\mathbf{H}-\mathbf{X}\|_{{\rm F}}^{2}+\gamma\|\mathbf{S}-\mathbf{A}\|_{{\rm F}}^{2}+\lambda\mathrm{Tr}\left(\mathbf{H}^{\top}\mathbf{L}_{{\rm rw}}\mathbf{H}\right)+\mu_{2}\left\Vert \mathbf{S}\right\Vert _{{\rm F}}^{2}.
\]
The $L$-smoothness of $f\left(\mathbf{H},\mathbf{S}\right)$ is deomnstrated
in the following lemma.
\begin{lem}
\label{lem:L-smooth-joint}Function $f\left(\mathbf{H},\mathbf{S}\right)$
is $L$-smooth with $L=\max\Bigl\{\sqrt{L_{H}^{2}+(1+\frac{1}{c}N\sqrt{N})^{2}\frac{4\lambda^{2}}{c^{2}}NB^{2}},$
$\sqrt{L_{S}^{2}+(1+\frac{1}{c}N^{2})^{2}\frac{4\lambda^{2}}{c^{2}}NB^{2}}\Bigr\}$,
i.e., for any $\mathbf{H}_{1},\mathbf{H}_{2}\in\mathbb{R}^{N\times M}$
and $\mathbf{S}_{1},\mathbf{S}_{2}\in\mathbb{R}^{N\times N}$, the
following inequality holds:
\[
\left\Vert \nabla f\left(\mathbf{H}_{1},\mathbf{S}_{1}\right)-\nabla f\left(\mathbf{H}_{2},\mathbf{S}_{2}\right)\right\Vert _{{\rm F}}\leq L\sqrt{\left\Vert \mathbf{H}_{1}-\mathbf{H}_{2}\right\Vert _{{\rm F}}^{2}+\left\Vert \mathbf{S}_{1}-\mathbf{S}_{2}\right\Vert _{{\rm F}}^{2}}.
\]
\end{lem}

\begin{proof}
Observe that
\begin{equation}
\begin{aligned} & \left\Vert \left[\begin{array}{c}
\nabla_{\mathbf{H}}f(\mathbf{H}_{1},\mathbf{S}_{1})\\
\nabla_{\mathbf{S}}f(\mathbf{H}_{1},\mathbf{S}_{1})
\end{array}\right]-\left[\begin{array}{c}
\nabla_{\mathbf{H}}f(\mathbf{H}_{2},\mathbf{S}_{2})\\
\nabla_{\mathbf{S}}f(\mathbf{H}_{2},\mathbf{S}_{2})
\end{array}\right]\right\Vert _{{\rm F}}^{2}\\
= & \left\Vert \nabla_{\mathbf{H}}f(\mathbf{H}_{1},\mathbf{S}_{1})-\nabla_{\mathbf{H}}f(\mathbf{H}_{2},\mathbf{S}_{2})\right\Vert _{{\rm F}}^{2}+\left\Vert \nabla_{\mathbf{S}}f(\mathbf{H}_{1},\mathbf{S}_{1})-\nabla_{\mathbf{S}}f(\mathbf{H}_{2},\mathbf{S}_{2})\right\Vert _{{\rm F}}^{2}\\
= & \left\Vert \nabla_{\mathbf{H}}f(\mathbf{H}_{1},\mathbf{S}_{1})-\nabla_{\mathbf{H}}f(\mathbf{H}_{1},\mathbf{S}_{2})+\nabla_{\mathbf{H}}f(\mathbf{H}_{1},\mathbf{S}_{2})-\nabla_{\mathbf{H}}f(\mathbf{H}_{2},\mathbf{S}_{2})\right\Vert _{{\rm F}}^{2}\\
 & \hspace{1.7cm}+\left\Vert \nabla_{\mathbf{S}}f(\mathbf{H}_{1},\mathbf{S}_{1})-\nabla_{\mathbf{S}}f(\mathbf{H}_{2},\mathbf{S}_{1})+\nabla_{\mathbf{S}}f(\mathbf{H}_{2},\mathbf{S}_{1})-\nabla_{\mathbf{S}}f(\mathbf{H}_{2},\mathbf{S}_{2})\right\Vert _{{\rm F}}^{2}\\
\leq & L_{H}^{2}\left\Vert \mathbf{H}_{1}-\mathbf{H}_{2}\right\Vert _{{\rm F}}^{2}+L_{S}^{2}\left\Vert \mathbf{S}_{1}-\mathbf{S}_{2}\right\Vert _{{\rm F}}^{2}\\
 & \hspace{1.7cm}+\left\Vert \nabla_{\mathbf{H}}f(\mathbf{H}_{1},\mathbf{S}_{1})-\nabla_{\mathbf{H}}f(\mathbf{H}_{1},\mathbf{S}_{2})\right\Vert _{{\rm F}}^{2}+\left\Vert \nabla_{\mathbf{S}}f(\mathbf{H}_{1},\mathbf{S}_{1})-\nabla_{\mathbf{S}}f(\mathbf{H}_{2},\mathbf{S}_{1})\right\Vert _{{\rm F}}^{2},
\end{aligned}
\label{eq:joint derivative}
\end{equation}
where the Lipschitz constants $L_{H}$ and $L_{S}$ are given in Lemma
\ref{lem:L-smooth-H} and Lemma \ref{lem:L-smooth-S}. 

In the following, we will derive the upper bound for the third and
the fourth term in Eq. \eqref{eq:joint derivative}. Based on Lemma
\ref{lem:D1-D2 bound}, the term $\left\Vert \nabla_{\mathbf{H}}f(\mathbf{H}_{1},\mathbf{S}_{1})-\nabla_{\mathbf{H}}f(\mathbf{H}_{1},\mathbf{S}_{2})\right\Vert _{{\rm F}}$
can be upperbounded as follows:
\begin{align}
 & \left\Vert \nabla_{\mathbf{H}}f(\mathbf{H}_{1},\mathbf{S}_{1})-\nabla_{\mathbf{H}}f(\mathbf{H}_{1},\mathbf{S}_{2})\right\Vert _{{\rm F}}\nonumber \\
= & \left\Vert 2\lambda\left(\mathbf{I}-\mathbf{D}_{1}^{-1}\mathbf{S}_{1}\right)\mathbf{H}_{1}-2\lambda\left(\mathbf{I}-\mathbf{D}_{2}^{-1}\mathbf{S}_{2}\right)\mathbf{H}_{1}\right\Vert _{{\rm F}}\\
\leq & 2\lambda\left\Vert \mathbf{D}_{1}^{-1}\mathbf{S}_{1}-\mathbf{D}_{2}^{-1}\mathbf{S}_{2}\right\Vert _{2}\left\Vert \mathbf{H}_{1}\right\Vert _{{\rm F}}\nonumber \\
\leq & \bigl(1+\frac{1}{c}N^{2}\bigr)\frac{2\lambda}{c}\sqrt{N}B\left\Vert \mathbf{S}_{1}-\mathbf{S}_{2}\right\Vert _{{\rm F}}.\label{eq:third term joint}
\end{align}
Besides, for the forth term, we have
\begin{align*}
 & \left\Vert \nabla_{\mathbf{S}}f(\mathbf{H}_{1},\mathbf{S}_{1})-\nabla_{\mathbf{S}}f(\mathbf{H}_{2},\mathbf{S}_{1})\right\Vert _{{\rm F}}\\
\leq & \left\Vert \lambda\mathbf{D}_{1}^{-1}\bigl(\mathbf{H}_{1}\mathbf{H}_{1}^{\top}-\mathbf{H}_{2}\mathbf{H}_{2}^{\top}\bigr)-\lambda\mathrm{Diag}\left(\mathbf{D}_{1}^{-1}\mathbf{S}_{1}\mathbf{H}_{1}\mathbf{H}_{1}^{\top}\mathbf{D}_{1}^{-1}-\mathbf{D}_{1}^{-1}\mathbf{S}_{1}\mathbf{H}_{2}\mathbf{H}_{2}^{\top}\mathbf{D}_{1}^{-1}\right)\mathbf{1}^{\top}\right\Vert _{{\rm F}}\\
\leq & \lambda\left\Vert \mathbf{D}_{1}^{-1}\right\Vert _{2}\left\Vert \mathbf{H}_{1}\mathbf{H}_{1}^{\top}-\mathbf{H}_{2}\mathbf{H}_{2}^{\top}\right\Vert _{{\rm F}}+\lambda\sqrt{N}\left\Vert \mathbf{D}_{1}^{-1}\mathbf{S}_{1}\right\Vert _{2}\left\Vert \mathbf{H}_{1}\mathbf{H}_{1}^{\top}-\mathbf{H}_{2}\mathbf{H}_{2}^{\top}\right\Vert _{{\rm F}}\left\Vert \mathbf{D}_{1}^{-1}\right\Vert _{2}.
\end{align*}
According to Lemma \ref{lem:DS bound}, we have
\[
\left\Vert \nabla_{\mathbf{S}}f(\mathbf{H}_{1},\mathbf{S}_{1})-\nabla_{\mathbf{S}}f(\mathbf{H}_{2},\mathbf{S}_{1})\right\Vert _{{\rm F}}\leq\bigl(1+\frac{1}{c}N\sqrt{N}\bigr)\frac{\lambda}{c}\left\Vert \mathbf{H}_{1}\mathbf{H}_{1}^{\top}-\mathbf{H}_{2}\mathbf{H}_{2}^{\top}\right\Vert _{{\rm F}}.
\]
Then, we can get
\begin{equation}
\begin{aligned} & \left\Vert \nabla_{\mathbf{S}}f(\mathbf{H}_{1},\mathbf{S}_{1})-\nabla_{\mathbf{S}}f(\mathbf{H}_{2},\mathbf{S}_{1})\right\Vert _{{\rm F}}\\
= & \bigl(1+\frac{1}{c}N\sqrt{N}\bigr)\frac{\lambda}{c}\left\Vert \mathbf{H}_{1}\mathbf{H}_{1}^{\top}-\mathbf{H}_{1}\mathbf{H}_{2}^{\top}+\mathbf{H}_{1}\mathbf{H}_{2}^{\top}-\mathbf{H}_{2}\mathbf{H}_{2}^{\top}\right\Vert _{{\rm F}}\\
\leq & \bigl(1+\frac{1}{c}N\sqrt{N}\bigr)\frac{\lambda}{c}\Bigl(\left\Vert \mathbf{H}_{1}\right\Vert _{2}\left\Vert \mathbf{H}_{1}-\mathbf{H}_{2}\right\Vert _{{\rm F}}+\left\Vert \mathbf{H}_{1}-\mathbf{H}_{2}\right\Vert _{{\rm F}}\left\Vert \mathbf{H}_{2}\right\Vert _{2}\Bigr)\\
\leq & \bigl(1+\frac{1}{c}N\sqrt{N}\bigr)\frac{\lambda}{c}\Bigl(\left\Vert \mathbf{H}_{1}\right\Vert _{{\rm F}}+\left\Vert \mathbf{H}_{2}\right\Vert _{{\rm F}}\Bigr)\left\Vert \mathbf{H}_{1}-\mathbf{H}_{2}\right\Vert _{{\rm F}}\\
\leq & \bigl(1+\frac{1}{c}N\sqrt{N}\bigr)\frac{2\lambda}{c}\sqrt{N}B\left\Vert \mathbf{H}_{1}-\mathbf{H}_{2}\right\Vert _{{\rm F}}.
\end{aligned}
\label{eq:fourth term joint}
\end{equation}
Substituting the results in Eq. \eqref{eq:third term joint} and Eq.
\eqref{eq:fourth term joint} into Eq. \eqref{eq:joint derivative}
gives
\begin{align*}
 & \left\Vert \left[\begin{array}{c}
\nabla_{\mathbf{H}}f(\mathbf{H}_{1},\mathbf{S}_{1})\\
\nabla_{\mathbf{S}}f(\mathbf{H}_{1},\mathbf{S}_{1})
\end{array}\right]-\left[\begin{array}{c}
\nabla_{\mathbf{H}}f(\mathbf{H}_{2},\mathbf{S}_{2})\\
\nabla_{\mathbf{S}}f(\mathbf{H}_{2},\mathbf{S}_{2})
\end{array}\right]\right\Vert _{{\rm F}}^{2}\\
\leq & \Bigl(L_{H}^{2}+\bigl(1+\frac{1}{c}N\sqrt{N}\bigr)^{2}\frac{4\lambda^{2}}{c^{2}}NB^{2}\Bigr)\left\Vert \mathbf{H}_{1}-\mathbf{H}_{2}\right\Vert _{{\rm F}}^{2}+\Bigl(L_{S}^{2}+\bigl(1+\frac{1}{c}N^{2}\bigr)^{2}\frac{4\lambda^{2}}{c^{2}}NB^{2}\Bigr)\left\Vert \mathbf{S}_{1}-\mathbf{S}_{2}\right\Vert _{{\rm F}}^{2}.
\end{align*}
Thus, function $f(\mathbf{H},\mathbf{S})$ is $L$-smooth with
\begin{equation}
L=\max\left\{ \sqrt{L_{H}^{2}+\bigl(1+\frac{1}{c}N\sqrt{N}\bigr)^{2}\frac{4\lambda^{2}}{c^{2}}NB^{2}},\sqrt{L_{S}^{2}+\bigl(1+\frac{1}{c}N^{2}\bigr)^{2}\frac{4\lambda^{2}}{c^{2}}NB^{2}}\right\} ,\label{eq.lconstant}
\end{equation}
through which the proof is completed.
\end{proof}
This result in Lemma \ref{lem:L-smooth-joint} indicates that the
Lipschitz constant $L$ is larger than both $L_{H}$ and $L_{S}$.
Practical graphs are commonly large graphs; i.e., the number of nodes
$N$ would dominate the other constants in $L$, i.e., $c$, $\lambda$,
and $B$. Therefore, if we use the joint optimization approach, the
Lipschitz constant is larger than $L_{H}$ and $L_{S}$ by a large
margin.

After deriving the Lipschitz constant in joint optimization, we compare
its convergence rate with the alternating optimization approach. Denote
$\{\mathbf{H}^{(k)},\mathbf{S}^{(k)}\}_{k=0}^{K}$ as the sequence
generated  by the above joint optimization approach. Following the
results in \citet{bolte2014proximal,nikolova2017alternating}, the
convergence property of the joint optimization approach is also given
in Eq. \eqref{eq:alternating-rate} with $\rho=\frac{1}{\eta}-\frac{L}{2}$.
Theoretically, to guarantee the sufficient descent of the objective
at each step, $\rho$ can be chosen to be $\min\{\frac{1}{\eta_{1}}-\frac{L_{H}}{2},\frac{1}{\eta_{2}}-\frac{L_{S}}{2}\}$
in the alternating optimization approach. Due to the fact that $L>\max\left\{ L_{H},L_{S}\right\} $
according to Eq. \eqref{eq.lconstant}, the alternating optimization
approach is allowed to adopt a larger step size at each block than
the joint optimization approach, resulting in a faster convergence
behavior of the sequence. Motivated by this fact, we develop ASMP
based on the alternating procedure rather than the joint one so that
the resulting message passing structure contains fewer layers to achieve
the similar or even better numerical performance compared to the joint
one.

\section{Convergence Property of ASMP in Practice\label{sec:Convergence-Property-of}}

To evaluate the convergence property of ASMP with learned step sizes,
we conduct experiments on Cora, Citeseer, and Cora-ML datasets at
a 25\% perturbation rate under meta-attack. Specially, we train a
4-layer ASGNN model and observe that the learned step sizes do not
satisfy the condition in Theorem \ref{ASMP: Convergence Theorem}.
In the following, we investigate the empirical convergence property
of ASMP. Since we use a recurrent structure in ASGNN, i.e., the step
sizes used in different layers are the same, we are able to extend
the trained 4-layer ASGNN model to a deeper one. The values of the
objective function in Problem \eqref{eq:Structure Signal Denoising}
in different layers are showcased in Figure \ref{fig:convergence},
where the objective values are normalized by dividing the objective
value in the first layer. From Figure \ref{fig:convergence}, we can
find that ASMP with learned step sizes can monotonically decrease
the objective function value during the message passing process. Note
that although the monotonic decreasing property does not hold in 16-18
layers on the Cora-ML dataset, this may be mainly due to the fact
that the step sizes are learned only based on a 4-layer model. The
results indicate that although the learned step sizes do not satisfy
the condition in Theorem \ref{ASMP: Convergence Theorem}, they still
ensure the monotonic decrease of the objective function value in practice.

\begin{figure}
\begin{centering}
\includegraphics[width=0.5\columnwidth]{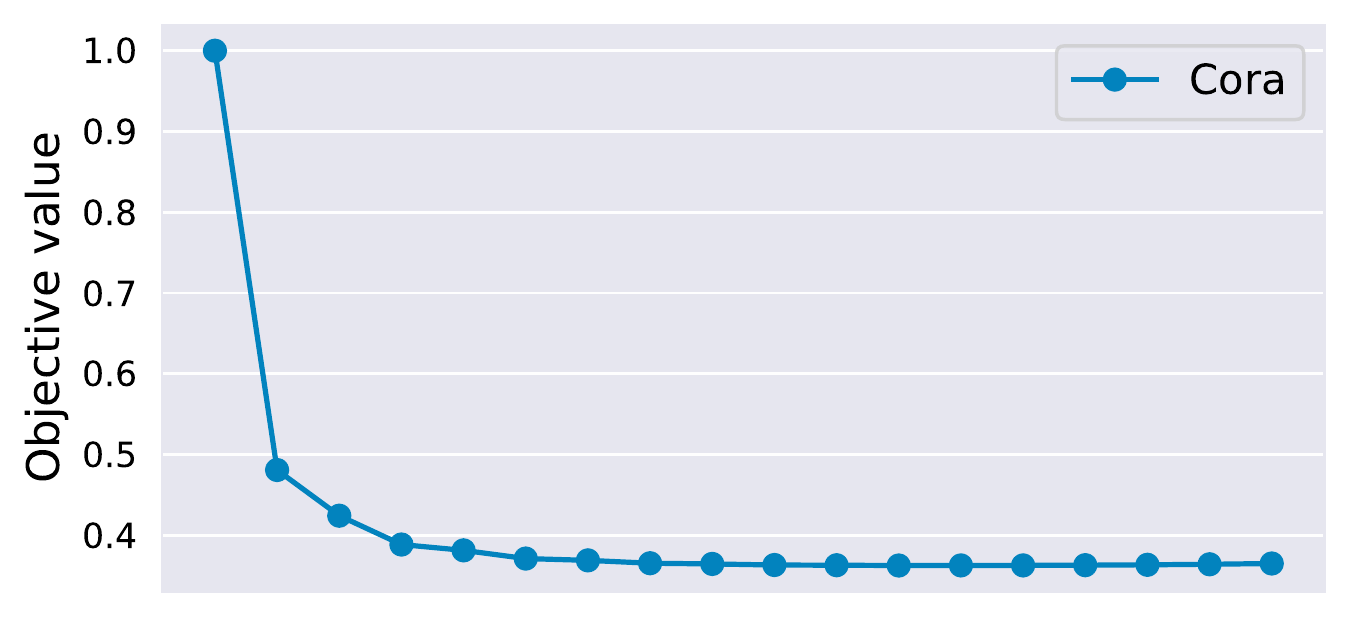} 
\par\end{centering}
\begin{centering}
\includegraphics[width=0.5\columnwidth]{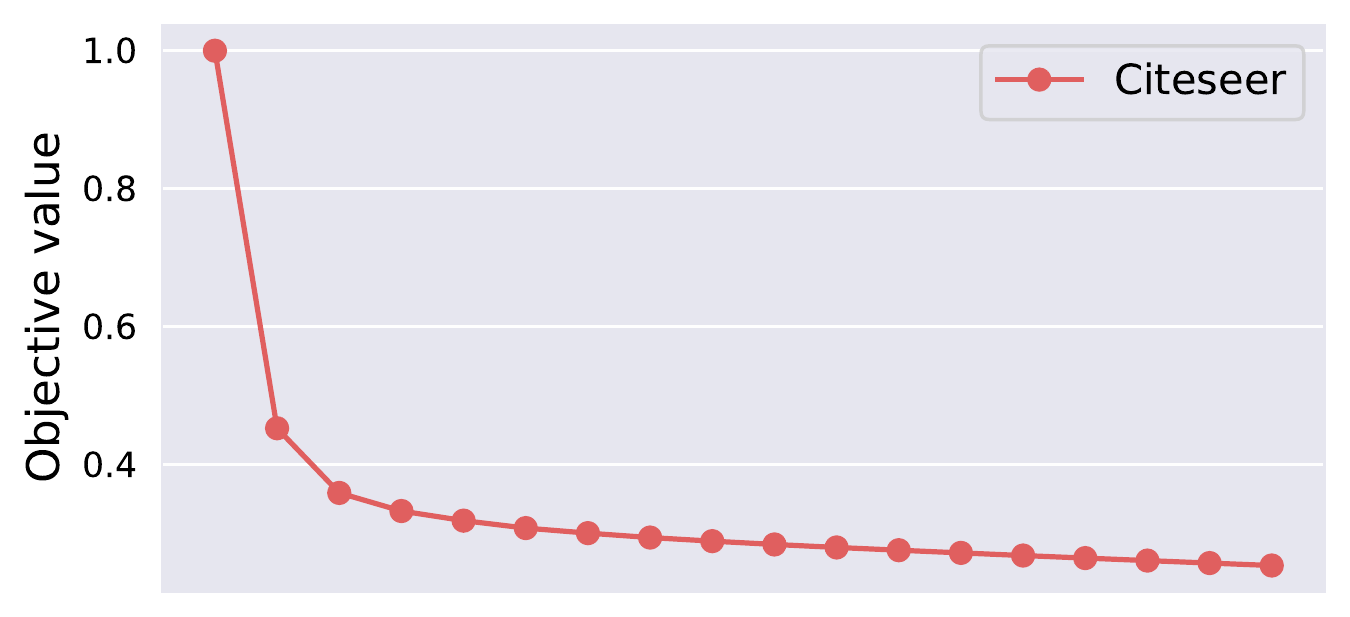} 
\par\end{centering}
\centering{}\includegraphics[width=0.5\columnwidth]{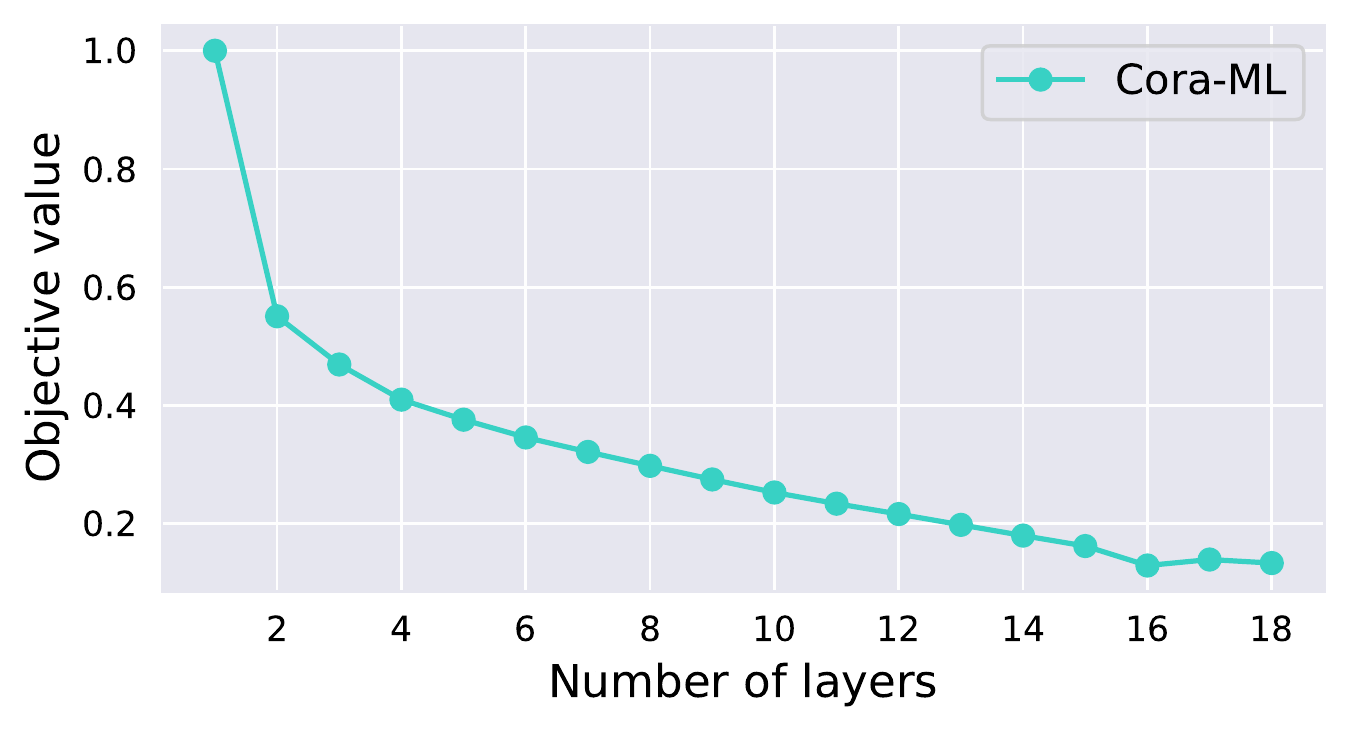}\caption{The objective value in Problem \eqref{eq:Structure Signal Denoising}
during ASMP.\label{fig:convergence}}
\end{figure}

\end{document}